\documentclass[10pt,twocolumn,letterpaper]{article}

\usepackage[pagenumbers]{cvpr} %

\usepackage{graphicx}
\usepackage{amsmath}
\usepackage{amssymb}
\usepackage{booktabs}

\usepackage{amsfonts,bm}
\usepackage{bbm}
\usepackage{thm-restate}
\usepackage{amsthm}
\usepackage{thmtools}

\newtheorem{theorem}{Theorem}

\newtheorem{lemma}{Lemma}

\theoremstyle{definition}
\newtheorem{definition}{Definition}

\usepackage[pagebackref,breaklinks,colorlinks]{hyperref}

\usepackage[capitalize]{cleveref}
\crefname{section}{Sec.}{Secs.}
\Crefname{section}{Section}{Sections}
\Crefname{table}{Table}{Tables}
\crefname{table}{Tab.}{Tabs.}

\def\cvprPaperID{12464} %
\def\confName{CVPR}
\def\confYear{2023}

\def\vx{{\bm{x}}}

\newcommand{\vdelta}{\bm{\delta}}

\DeclareMathOperator*{\argmax}{arg\,max}

\begin{document}

\title{A Practical Upper Bound for the Worst-Case Attribution Deviations}

\author{Fan Wang\quad\quad Adams Wai-Kin Kong \vspace{0.2cm} \\
Nanyang Technological University, Singapore\\
{\tt\small fan005@e.ntu.edu.sg\quad\quad adamskong@ntu.edu.sg}
}
\maketitle

\begin{abstract}
    Model attribution is a critical component of deep neural networks (DNNs) for its interpretability to complex models. Recent studies bring up attention to the security of attribution methods as they are vulnerable to attribution attacks that generate similar images with dramatically different attributions. Existing works have been investigating empirically improving the robustness of DNNs against those attacks; however, none of them explicitly quantifies the actual deviations of attributions. In this work, for the first time, a constrained optimization problem is formulated to derive an upper bound that measures the largest dissimilarity of attributions after the samples are perturbed by any noises within a certain region while the classification results remain the same. Based on the formulation, different practical approaches are introduced to bound the attributions above using Euclidean distance and cosine similarity under both $\ell_2$ and $\ell_\infty$-norm perturbations constraints. The bounds developed by our theoretical study are validated on various datasets and two different types of attacks (PGD attack and IFIA attribution attack). Over 10 million attacks in the experiments indicate that the proposed upper bounds effectively quantify the robustness of models based on the worst-case attribution dissimilarities. 
\end{abstract}

\section{Introduction}\label{sec:introduction}
Attribution methods play an important role in deep learning applications as one of the subareas of explainable AI. Practitioners use attribution methods to measure the relative importance among different features and to understand the impacts of features contributing to the model outputs. They have been widely used in a number of critical real-world applications, such as risk management \cite{bhatt2020explainable}, medical imaging \cite{sayres2019using,singh2020explainable} and drug discovery \cite{jimenez2020drug}. In particular, attributions are supposed to be secure and resistant to external manipulation such that proper explanations can be applied to safety-sensitive applications. Regulations are also deployed in countries to enforce the interpretability of deep learning models for a `right to explain' \cite{goodman2017european}. Although attribution methods have been extensively studied \cite{simonyan2013deep,zeiler2014visualizing,lundberg2017unified,shrikumar2017learning,sundararajan2017axiomatic,zintgraf2017visualizing}, recent works reveal that they are vulnerable to visually imperceptible perturbations that drastically alter the attributions and keep the model outputs unchanged \cite{ghorbani2019interpretation,dombrowski2019explanations}.

Prior works \cite{chen2019robust,boopathy2020proper,ivankay2020far,singh2020attributional,wang2020smoothed,sarkar2021enhanced,wang2022exploiting} investigate the attribution robustness based on empirical and statistical estimations over entire dataset. However, current attribution robustness works are unable to evaluate how robust the model is given any arbitrary test point, perturbed or unperturbed. In this paper, we study the problem of finding the worst attribution perturbation within certain predefined regions. Specifically, given a trained model and an image sample, we propose theoretical upper bounds of the attribution deviations from the unperturbed ones. As far as we know, this is the first attempt to provide an upper bound of attribution differences.

In this paper, the general upper bound for attribution deviation is first quantified as the maximum changes of attributions after the samples are perturbed while classification results remain the same. Two cases are analyzed, including with and without label constraint, which refers to the classification labels being unchanged and changed, respectively, after the original samples are attacked. For each case, two mostly used perturbation constraints, $\ell_2$ and $\ell_\infty$-norm, are considered to compute the upper bound. For $\ell_2$-norm constraint, our approach is based on the first-order Taylor series of model attribution, and a tight upper bound ignoring the label constraint is computed from the singular value of the attribution gradient. $\ell_\infty$-norm constraint is more complicated because the upper bound is a solution of a concave quadratic programming with box constraints, which is an NP-hard problem. Thus, two relaxation approaches are proposed. Moreover, a more restricted bound constrained on the unchanged label is also studied. In this study, Euclidean distance and cosine distance, which are also employed in the previous empirical studies \cite{chen2019robust,singh2020attributional,wang2022exploiting}, are used as dissimilarity functions to measure attribution difference. We summarize the contributions of this paper as follows:
\begin{itemize}
  \item We formally define the general upper bound for attribution deviation as the optimization problem with norm constraint and label constraint to find the maximum change of attributions after samples being attacked. According to the best knowledge of the authors, it has not been studied before.
  \item The tight upper bounds for $\ell_2$-norm constrained attacks with and without classification label constraints are proposed based on the first-order Taylor series. The proposed bound without label constraints generalizes to all gradient-based attribution methods, and the one with label constraints is applicable to all attribution methods satisfying the axiom of completeness.
  \item Two different approaches are provided to bound the $\ell_\infty$-norm constrained attacks above, which uses an $\ell_p$-norm relaxation and a mathematical property of the quadratic form. 
  \item The experimental results show that the upper bounds derived in this paper can effectively bound the attribution differences between all 10 million attacked samples and their corresponding original samples from different models, datasets, attack methods and parameters choices. 
\end{itemize}

The rest of this paper is organized as follows. We start with an introduction to notations and related works. The formulation of the general upper bound for attribution deviation is defined in Sec.~\ref{sec:formulation}. Specific methods to find the upper bounds in different scenarios are provided in Sec.~\ref{sec:upper_bound}. In Sec.~\ref{sec:exp}, detailed experimental results are presented and the paper concludes in Sec.~\ref{sec:conclusion}.

\section{Preliminaries and related works}\label{sec:related}
We consider a twice-differentiable classifier $f$ that maps the input set $\mathcal{D}=\left\{(\vx^{(i)}, y^{(i)})\right\}_{i=1}^n$ to the logits, $f:\mathbb{R}^d\rightarrow\mathbb{R}^k$, where $\vx^{(i)}\in\mathbb{R}^d$ and $y^{(i)}\in\left\{1, \ldots,k\right\}$ represent the $i$-th sample and its ground truth label. The non-bold version $x_k$ represents the $k$-th feature of $\vx$ and $f_y$ is the logit at label $y$. The model attribution of the input sample given label $y$ is computed by $g^y:\mathbb{R}^d\rightarrow\mathbb{R}^d$, and we denote the attribution of $\vx$ by $g^{y}(\vx)$.

\subsection{Model attribution}
The model attribution studies the importance of each input feature $x_i$ that contributes to the final output $f_y(\vx)$. We can classify the most used attribution methods into two categories, perturbations-based methods \cite{zeiler2014visualizing,zintgraf2017visualizing} and backpropagation-based methods \cite{shrikumar2017learning,bach2015pixel}, which include gradient-based methods. In particular, in this paper, we focus on the most commonly used gradient-based attribution methods, saliency map (SM), gradient*input and integrated gradient (IG). Saliency map \cite{simonyan2013deep} is defined as the gradients of output with respect to the input. Gradient*input \cite{shrikumar2016not} is computed by element-wise multiplication of input features and the gradients. Integrated gradients \cite{sundararajan2017axiomatic} is defined as line integral of gradients from a baseline image $\bm{a}$ to the input image $\vx$ weighted by their difference\footnote{The baseline is chosen to be a black image ($\bm{a}=\bm{0}$) in this paper if not specifically stated. Without loss of generality, $f_y(\bm{a})=0$.}. It is worth noting that IG satisfies the axiom of completeness, $\sum_i g_i^y(\vx) = f_y(\vx)$, which builds a direct connection between the attributions and model outputs. The mathematical expressions and examples of the attribution methods are given in Table~\ref{tbl:attr_examples}.

\begin{table*}
  \centering
  \caption{Mathematical expressions and visual examples of the selected attribution methods. Attributions have been taken absolute values and are presented heatmaps to reflect relative importance among pixels. The baseline $\bm{a}$ of IG is chosen as a black image. $\otimes$ denotes the element-wise multiplication.}\label{tbl:attr_examples}
  \resizebox{\textwidth}{!}{%
  \begin{tabular}{cccc}
    \toprule
    Original image &Saliency map &Input*gradient &Integrated gradients\\
    &$\nabla f_y(\vx)$&$\vx \otimes \nabla f_y(\vx)$&$(\vx - \bm{a})\otimes\int_0^1\nabla f_y(\bm{a} + \alpha(\vx - \bm{a}))\,d\alpha$\\
    \midrule
    \includegraphics[width=.3\textwidth]{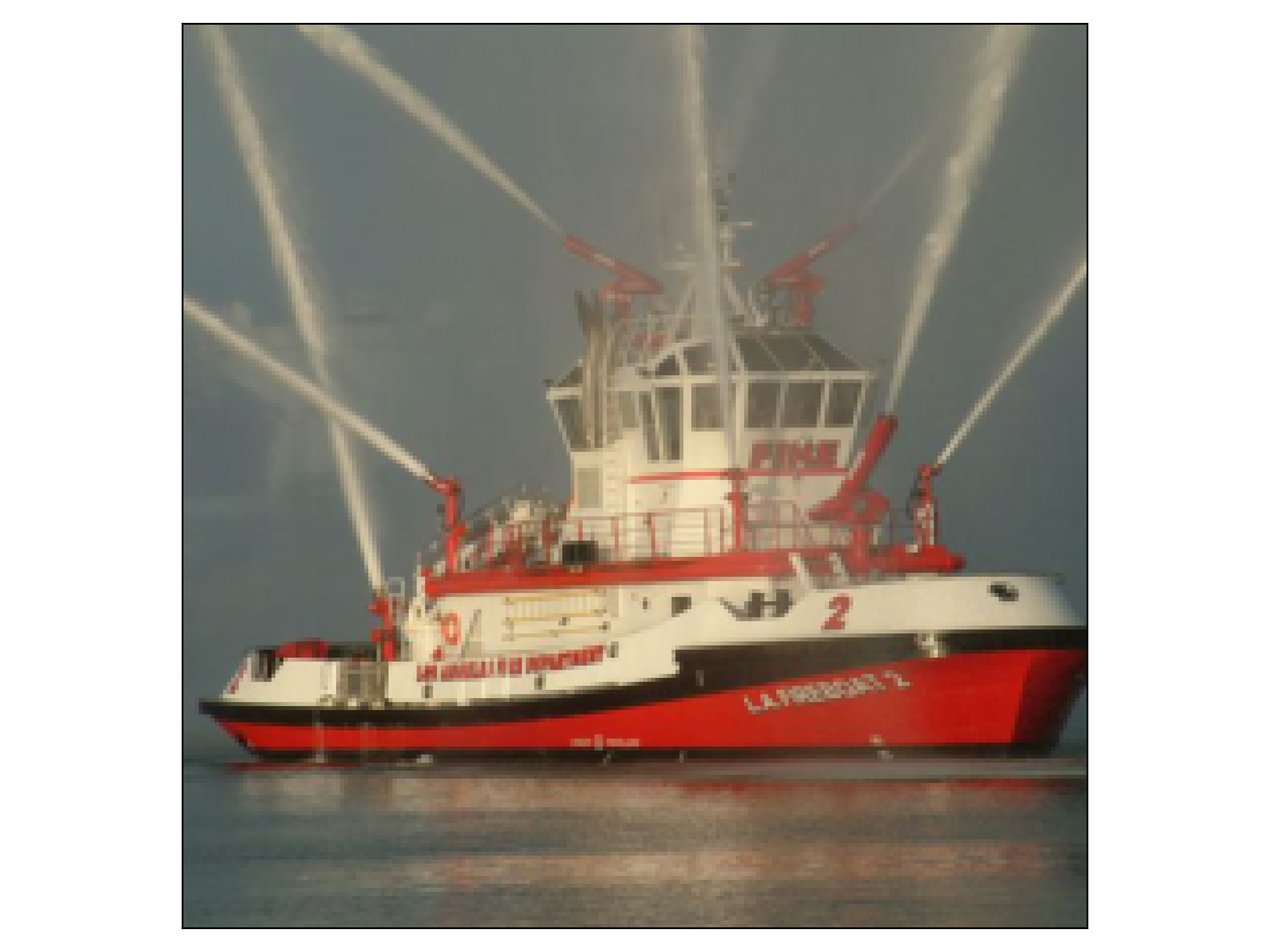}&\includegraphics[width=.3\textwidth]{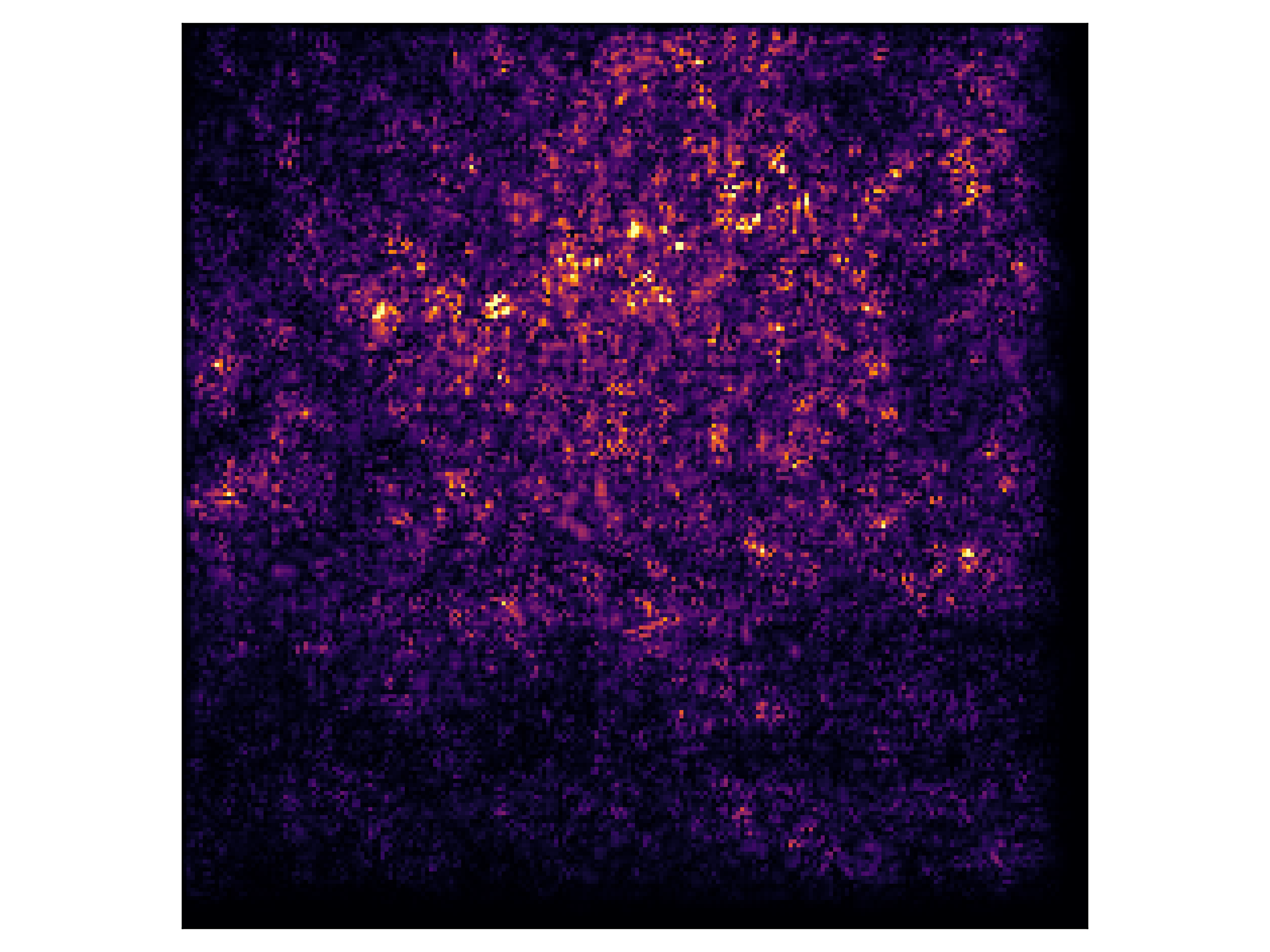}&\includegraphics[width=.3\textwidth]{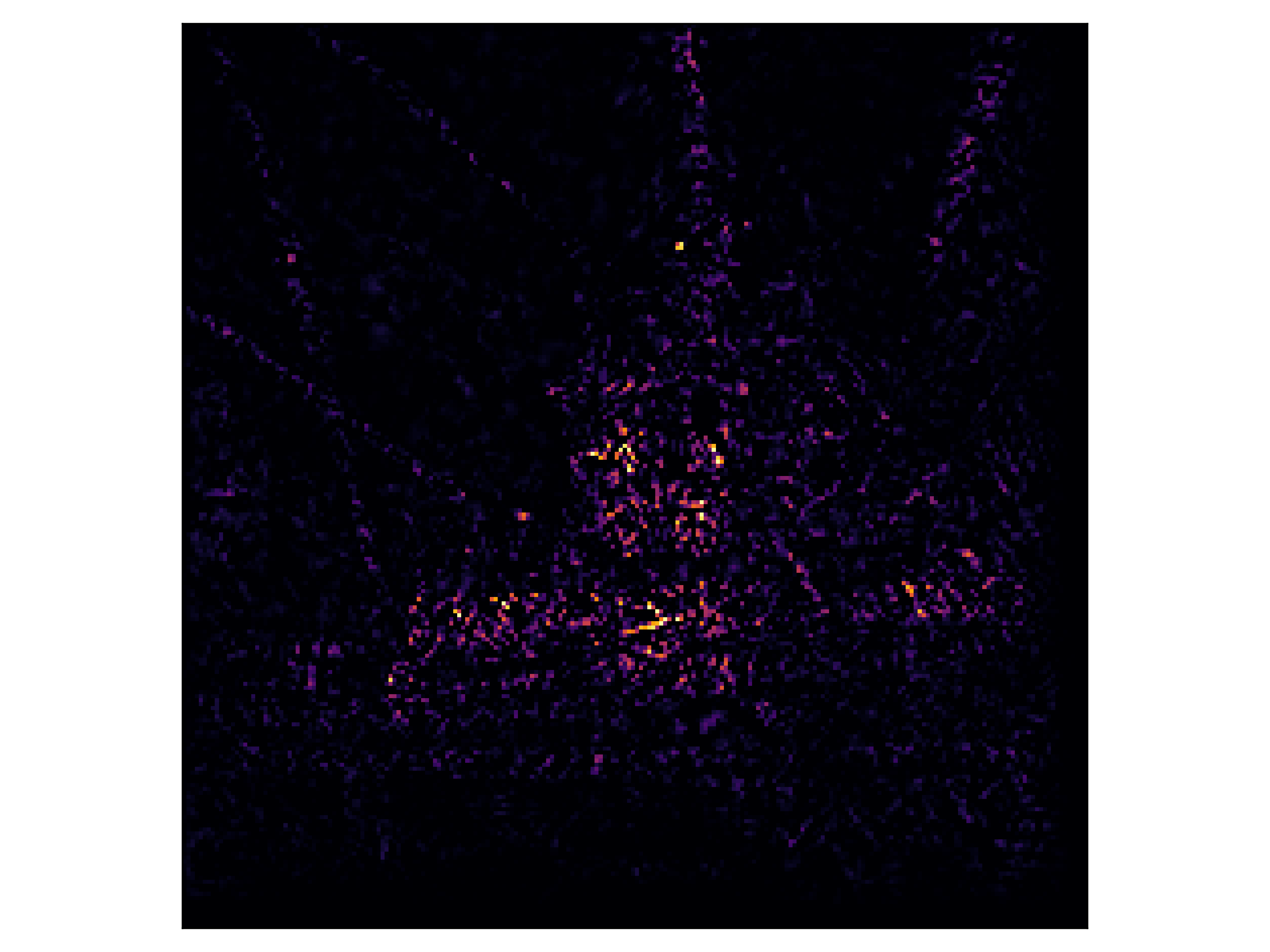}&\includegraphics[width=.3\textwidth]{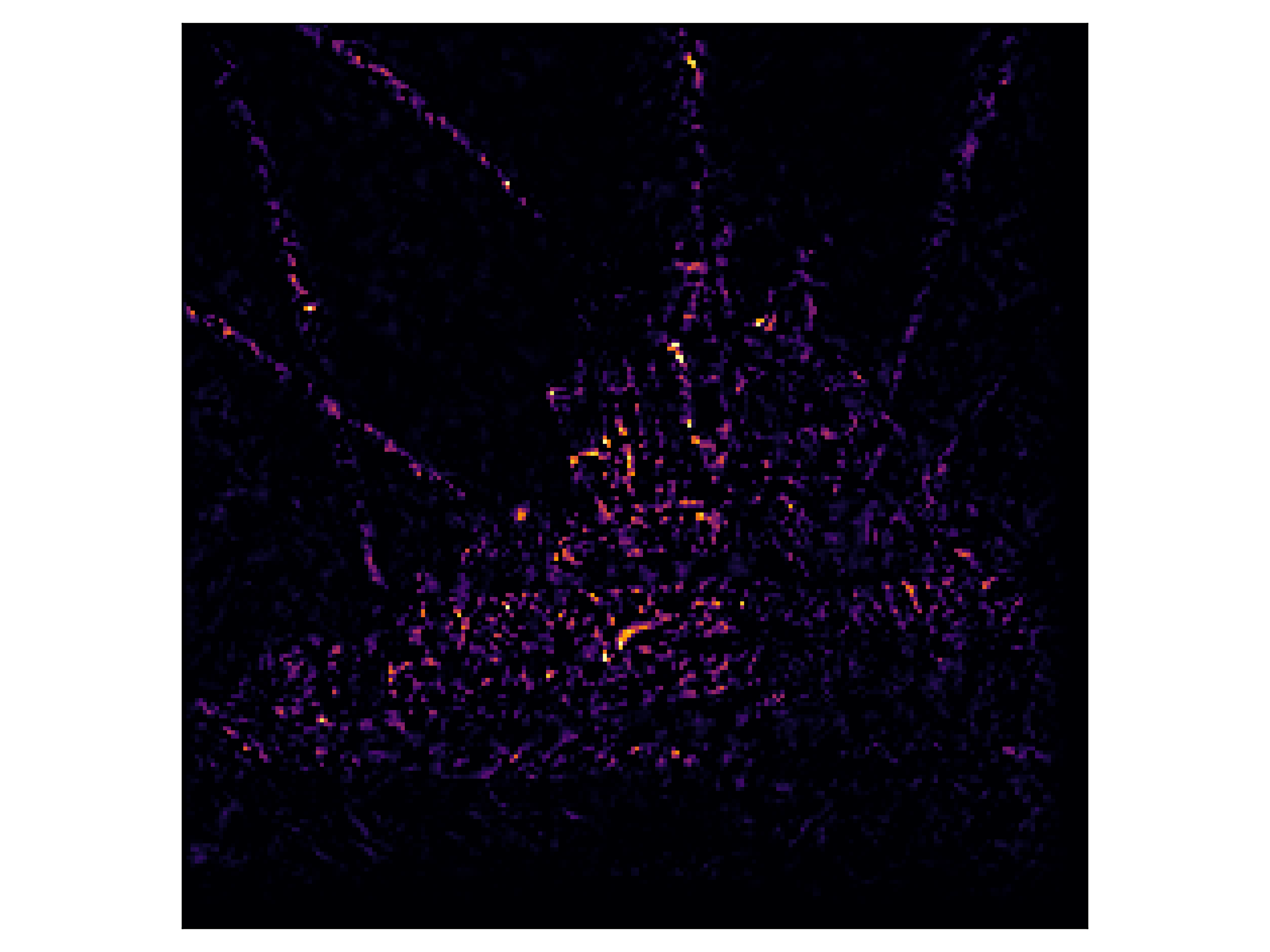}\\
    \bottomrule
  \end{tabular}
  }
\end{table*}

\subsection{Attribution robustness}
It has been discovered in the literature that model attributions can be easily sabotaged by adversaries. Similar to adversarial examples \cite{goodfellow2014explaining}, human-indistinguishable perturbations can be also augmented to natural images that, though classification results remain unchanged, misdirect the model attributions towards meaningless interpretations \cite{ghorbani2019interpretation} or any predefined arbitrary patterns which are unrelated to the original images \cite{dombrowski2019explanations}. 

To mitigate the threat of being attacked, researchers have also worked on training attribution robust models. The most considered techniques are adapted from adversarial training \cite{madry2018towards}, and they minimize the differences between original and the worst-case perturbed attributions. 
Boopathy \etal \cite{boopathy2020proper} and Chen \etal \cite{chen2019robust} consider the $\ell_1$-norm distance to measure the difference between attributions, and Ivankay \etal \cite{ivankay2020far} uses Pearson correlation coefficient. Singh \etal \cite{singh2020attributional} and Wang \etal \cite{wang2020smoothed} choose $\ell_2$-norm distance, where the former minimizes the spatial correlation between image and attribution using a soft-margin triplet loss, and the latter shows the smoothness of the decision surface is related to attribution robustness based on a geometric understanding. Wang \& Kong \cite{wang2022exploiting} emphasizes the directions of attributions using the relationship between Kendall's rank correlation and cosine similarity and protects the attribution based on the latter. However, none of the aforementioned methods defines a quantitative measurement to the attribution changes after perturbation. More clearly, the attributions are not guaranteed to be protected for all perturbations within the allowable region that do not alter the classification outputs.

\section{Formulation of general upper bound for attribution deviations}\label{sec:formulation}
In this section, we formally define the general upper bound for attribution deviation as a measurement of attribution robustness. Recall that adversaries incapacitate the attributions of neural networks by adding imperceptible noises to natural images. For an attributional robust model, on the contrary, the imperceptible noises should not change the interpretability of attributions, \emph{i.e.}, images perturbed by noises should provide \emph{similar} attributions as the original ones. To evaluate such resistance against adversaries, it is essential to find an upper bound that represents the worst-case dissimilarity of attributions after the original images being perturbed. Thus, we define the general upper bound for attribution deviations as follows.

\begin{definition}[General upper bound for attribution deviations]\label{def:cert_ar}
  Given a trained neural network $f$, a fixed allowable region for perturbation $\vdelta$, $\mathcal{B}_{\varepsilon} = \left\{\vdelta: \Vert\vdelta\Vert_p\leq\varepsilon\right\}$, and an input sample $\vx$, the \emph{general upper bound for attribution deviations} $T(\varepsilon;\vx)$ is defined that, for all perturbations $\vdelta\in\mathcal{B}_\varepsilon$, if $\argmax_k f_k(\vx) = \argmax_k f_k(\vx+\vdelta)$, then their corresponding attributions satisfy that $D(g^y(\vx), g^y(\vx+\vdelta))\leq T(\varepsilon;\vx)$.
\end{definition}

In the above definition, $D(\cdot, \cdot)$ is a dissimilarity metric that measures the difference between two attributions, where a smaller value indicates that two attributions are more similar and represent closer meaningful interpretations. $T$ is a function with respect to the threshold $\varepsilon$ and $\vx$. The definition formalizes the guarantee of attribution robustness, where when the model is more robust, the model attributions being attacked are less likely to be misled. More precisely, when being attacked, the change of attribution is bounded above and the smaller upper bound indicates the more attributional robust model.

Based on the above definition, the general upper bound for attribution deviations $T(\varepsilon;\vx)$ can be found by solving the following optimization problem 
\begin{equation}\label{eq:opt_cert_ar}
  \begin{aligned}
    \max_{\vdelta} \qquad & D(g^y(\vx), g^y(\vx+\vdelta))\\
    \text{s.t.} \qquad &\Vert\vdelta\Vert_p\leq\varepsilon\\
    & \argmax_k f_k(\vx) = \argmax_k f_k(\vx+\vdelta)
  \end{aligned}
\end{equation}

We refer the first constraint $\Vert\vdelta\Vert_p\leq\varepsilon$ to the norm constraint and the second one to the label constraint as it requires the unchanged label after being perturbed. In following sections, we attempt to solve the optimization problem (\ref{eq:opt_cert_ar}) using the two mostly used norm constraints on the perturbations, $\ell_2$ and $\ell_\infty$, \emph{i.e.}, $\Vert\vdelta\Vert_2\leq\varepsilon$ and $\Vert\vdelta\Vert_\infty\leq\varepsilon$. For the dissimilarity metric, we choose from previously used attribution measurements, the Euclidean distance and the cosine distance. An alternative formulation of this problem is to find the maximum perturbation $\varepsilon$ subject to $D(g^y(\vx), g^y(\vx+\vdelta))\leq \omega$ where $\omega$ is a predefined threshold. The results obtained from problem (\ref{eq:opt_cert_ar}) can be converted directly to the alternative formulation, and we defer the procedures to Appendix~\ref{append:alt_form}.

\section{A practical upper bound for attribution deviations}\label{sec:upper_bound}
\subsection{Upper bound under \texorpdfstring{$\ell_2$}{L2}-norm constraint without the label constraint}\label{sec:l2_nonlabel}
To start with, the upper bound without label constraints is studied. This will provide a looser bound since, intuitively, stronger adversaries are allowed to perturb the samples that may change the classification results. While perturbations are restricted in a small region and indistinguishable to human, they could be enough to cause huge difference in attributions. The upper bound for $\ell_2$-norm constrained case is a straightforward derivation of the first-order Taylor series of attribution functions. The following theorem provides a tight bound for attribution robustness assuming that the attribution function is locally linear.
\begin{restatable}{theorem}{thmltwounlabel}\label{thm:l2_unlabel}
  Given a twice-differentiable classifier $f: \mathbb{R}^d\rightarrow\mathbb{R}^k$, and its attribution $g^y$ on label $y$, assume that $g^y$ is locally linear within the neighborhood of $\vx$, $\mathcal{B}_\varepsilon(\vx)=\left\{\vx+\vdelta\vert\Vert\vdelta\Vert_2\leq\varepsilon\right\}$, then for all perturbations $\Vert\vdelta\Vert_2\leq\varepsilon$, 
  $$\Vert g^y(\vx+\vdelta)-g^y(\vx)\Vert_2\leq \xi_{max}\varepsilon,$$ 
  where $\xi_{max}$ is the largest singular value of $H=\nabla g^y(\vx)$.
\end{restatable}

\begin{proof}
  Based on the Taylor series of $g^y(\vx)$ and the above condition, we have
  \begin{align}
    \Vert g^y(\vx+\vdelta) - g^y(\vx) \Vert_2^2 &\leq \Vert \vdelta^\top\nabla g^y(\vx)\Vert_2^2\\
    &= \vdelta^\top\nabla g^y(\vx)\nabla g^y(\vx)^\top\vdelta\\
    &= \frac{\vdelta^\top}{\Vert\vdelta\Vert_2}P \frac{\vdelta}{\Vert\vdelta\Vert_2} \cdot \Vert\vdelta\Vert^2_2 \\
    &\leq \lambda_{max}\Vert\vdelta\Vert^2_2\leq \lambda_{max}\varepsilon^2 \label{eqn:eigen_ineq}
  \end{align}
  where $\lambda_{max}$ is the largest eigenvalue of $P = HH^\top = \nabla g^y(\vx)\nabla g^y(\vx)^\top$, and $\bm{v}_{max}$ is the corresponding eigenvector. The equality in Eq.~(\ref{eqn:eigen_ineq}) is achieved when $\vdelta$ is $\varepsilon\bm{v}_{max}$ or $-\varepsilon\bm{v}_{max}$. Since the singular values of $H$ are equal to the square root of the eigenvalues of $P$, then,
  \begin{equation}
    \Vert g^y(\vx+\vdelta) - g^y(\vx) \Vert_2 \leq \sqrt{\lambda_{max}}\varepsilon = \xi_{max}\varepsilon. \label{eqn:ub_l2}
  \end{equation}
\end{proof}

Note that the local linearity of attribution function is a weak assumption for both attribution and adversarial robust models since most of the defense methods \cite{qin2019adversarial,wang2020smoothed} attempt to smoothen the functions. In addition, when the magnitude of perturbation $\vdelta$ is constrained to small size, the magnitude of the higher-order Taylor remainders is negligible. We include the empirical results evaluating this assumption in Appendix~\ref{append:eval_locallinear}. Furthermore, we also provide a generalization of the theorem that bounds the attribution differences as a function of a constant $c\geq 1$ that measures the error margin of the first-order Taylor series in Appendix~\ref{append:general_thm}, which can be applied similarly to all other results in this work.

It is also noticed that the above theorem uses the gradient of attribution $H=\nabla g^y(\vx)$, which is also the Hessian matrix $\nabla^2 f_y(\vx)$ when the attribution is chosen as saliency maps and can be computed easily for other gradient-based attribution methods. Moreover, the second-order derivatives can be zeros for ReLU networks \cite{dombrowski2019explanations}. In this work, the non-linearity functions are replaced by softplus function $f(\vx;\beta) = \frac{1}{\beta}\log (1+e^{\beta\vx})$ as in Dombrowski \etal \cite{dombrowski2019explanations}. A 2D example of the upper bound is illustrated in Fig.~\ref{fig:illustrate}. The optimum solution is in the same direction as the semi-major axis of the ellipse, which represents $\vdelta^\top P\vdelta$. The circle represents the 2D Euclidean ball bounded by $T(\varepsilon;\vx)$, which is derived from the length of the semi-major axis. 

\begin{figure*}
  \centering
  \begin{subfigure}{.45\textwidth}
      \centering
      \includegraphics[width=\textwidth]{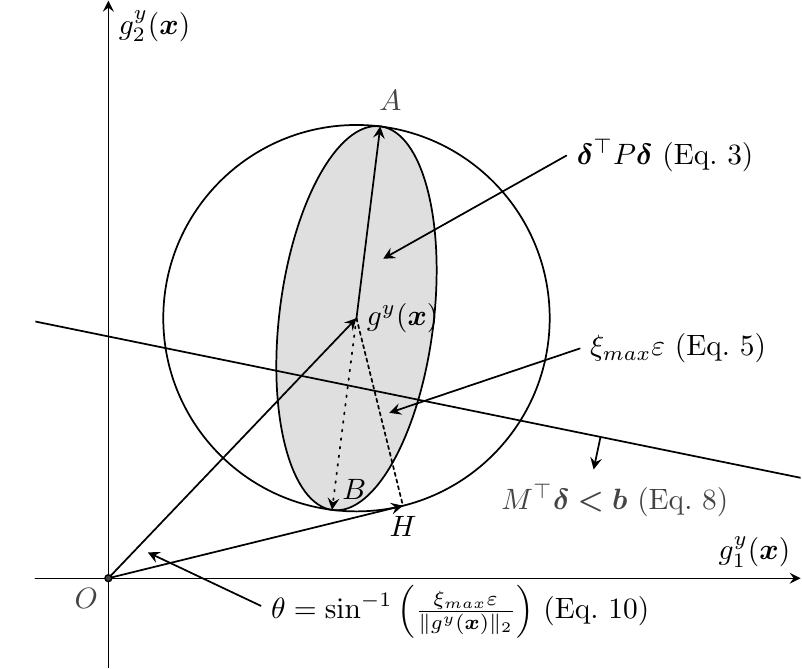}
      \caption{}\label{fig:illustrate}
  \end{subfigure}
  \hfill
  \begin{subfigure}{.54\textwidth}
    \centering
    \includegraphics[width=.75\textwidth]{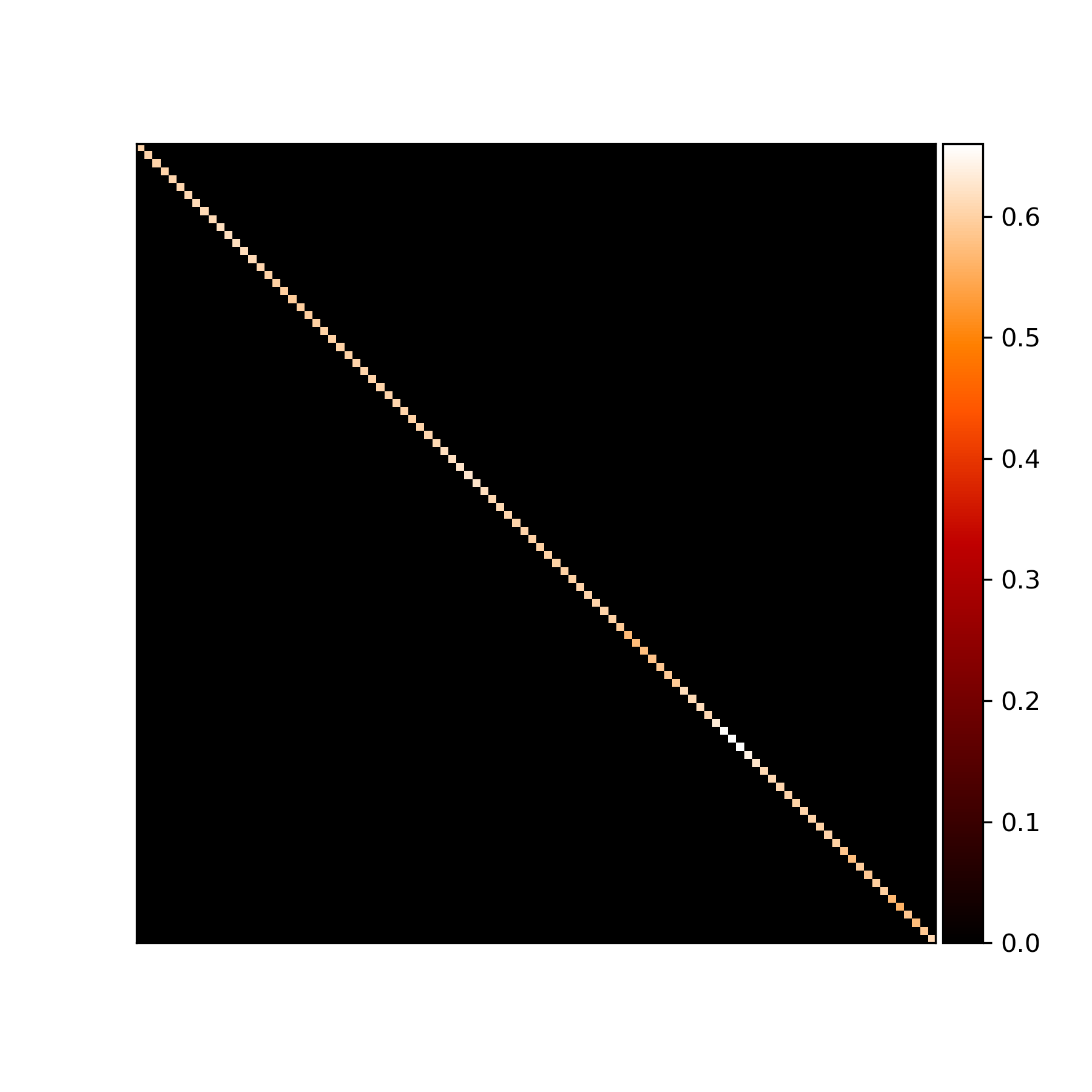}
    \caption{}\label{fig:hess_diag}
  \end{subfigure}
  \caption{(a) 2D illustration of the proposed upper bound on Euclidean distance and cosine distance. (b) Visualization of the absolute values of gradient IG as a heat map. The gradient is generated using CIFAR-10 ($3072\times 3072$), and the values are normalized to $[0, 1]$. Here the first 100 dimensions of each axis are plotted for better visualization. More figures and the mathematical analysis are given in Appendix~\ref{append:analyse_attr_grad}.}
\end{figure*}

\subsection{Upper bound under \texorpdfstring{$\ell_\infty$}{Linf}-norm constraint without the label constraint}\label{sec:linf_unlabel}
The upper bound for $\ell_\infty$-norm constrained case is more complicated as $\Vert\vdelta\Vert_\infty\leq\varepsilon$ defines a box constraints inequality system that $-\varepsilon\leq\delta_i\leq\varepsilon$ for all $i$. If we still consider the quadratic form derived from the first-order Taylor series as in Sec.~\ref{sec:l2_nonlabel}, the above optimization problem (\ref{eq:opt_cert_ar}) turns into a concave quadratic programming with box constraints, which is NP-hard \cite{pardalos1991quadratic}. In order to compute the upper bound efficiently, we consider a loose relaxation of $p$-norms.

\begin{restatable}{corollary}{linfunlabel}\label{cor:linf_unlabel}
  Given a twice-differentiable classifier $f:\mathbb{R}^d\rightarrow\mathbb{R}^k$, and its attribution $g^y$ on label $y$, assume that $g^y$ is locally linear within the neighborhood of $\vx$, $\mathcal{B}_\varepsilon(\vx)=\left\{\vx+\vdelta\vert\Vert\vdelta\Vert_p\leq\varepsilon\right\}$, then for all perturbations $\Vert\vdelta\Vert_p\leq\varepsilon$ that $p>2$, $\Vert g^y(\vx+\vdelta) - g^y(\vx) \Vert_2\leq d^{\frac{1}{2}-\frac{1}{p}}\xi_{max}\varepsilon$, where $\xi_{max}$ is the largest singular value of $H=\nabla g^y(\vx)$. %
\end{restatable}
The proof of the relaxation of $p$-norm and Corollary~\ref{cor:linf_unlabel} can be found in Appendix~\ref{proof:cor_linf_unlabel}. Note that this corollary not only avoids the NP-hard problem for $\ell_\infty$-norm constraint, but it is also a general upper bound for $p$-norm constraint on $\vdelta$ for all $p>2$. However, it is also noticed that the upper bound increases with respect to the input sample dimension. The multiplication factor for $\ell_\infty$ is $\sqrt{d}$. For high-dimensional input samples, the provided method would scale up to an extremely loose upper bound that can be trivial but meaningless. To better bound the attribution deviations in the $\ell_\infty$-norm case, we provide a tighter upper bound using the sparsity of attribution gradients. 

\begin{restatable}{theorem}{thmlinfunlabel}\label{thm:linf_unlabel}
  Given a twice-differentiable classifier $f$, its attribution on label $y$, $g^y$, and the gradient $H = \nabla g^y$, assume that $g^y$ is locally linear within the neighborhood of $\vx$, $\mathcal{B}_\varepsilon(\vx)=\left\{\vx+\vdelta\vert\Vert\vdelta\Vert_\infty\leq\varepsilon\right\}$, then for all perturbations $\Vert\vdelta\Vert_\infty\leq\varepsilon$,
  \begin{equation}\label{eqn:linf_sum_ub}
    \Vert g^y(\vx+\vdelta) - g^y(\vx) \Vert_2 \leq \varepsilon\sqrt{\sum_{i,j}\left\vert P_{ij}\right\vert}.
  \end{equation}
  where $P=HH^\top$ and the equality is taken at $\vdelta = (\pm\varepsilon, \ldots, \pm\varepsilon)^\top$.
\end{restatable}
The proof is deferred to Appendix~\ref{proof:thm_linf_unlabel}. This upper bound as the summation of absolute values of the matrix $P=\nabla g^y(\vx)\nabla g^y(\vx)^\top$ is shown to be tighter than that given in Corollary~\ref{cor:linf_unlabel} since $P$ is a diagonal-dominated and positive semi-definite matrix (see Fig.~\ref{fig:hess_diag}), which implies that $\left\vert P_{ii}\right\vert\approx\lambda_i$. 

\subsection{Upper bound with the label constraint}\label{sec:cert_labeled}
In this section, we extend our study of the upper bound to the case that labels are not changed after the samples are perturbed. Here, only attribution methods satisfying the axiom of completeness are studied as the axiom provides a direct connection between attributions and model outputs, \emph{i.e.}, $\sum_i g_i^y(\vx) = f_y(\vx)$. The following proposition gives a sufficient condition to ensure that the classification result remains unchanged after the sample is perturbed.
\begin{restatable}{proposition}{propcondlabel}\label{prop:cond_label}
  Denote the gradient-based attribution satisfying the completeness axiom of $\vx$ on ground truth label $y$ by $g^y(\vx)$, and the attribution on a different label $y'$ by $g^{y'}(\vx)$. Given the perturbation $\vdelta$, assume that $g^y$ is locally linear within the neighborhood of $\vx$, $\mathcal{B}_\varepsilon(\vx)=\left\{\vx+\vdelta\vert\Vert\vdelta\Vert_p\leq\varepsilon\right\}$, the classification result of $\vx+\vdelta$ does not change from $y$ to $y'$ if
  \begin{equation}
    \left(\left(\nabla g^{y'}(\vx) - \nabla g^{y}(\vx)\right)\Delta\right)^\top \vdelta < f_{y}(\vx) - f_{y'}(\vx),
  \end{equation}
  where $\Delta$ is an all one vector, $\Delta = \left(1,\ldots, 1\right)^\top\in\mathbb{R}^d$.
\end{restatable}

The full proof can be found in Appendix~\ref{proof:prop_cond_label}. Note that the inequality is linear to $\vdelta$ and we denote $M = (\nabla g^{y'}(\vx) - \nabla g^{y}(\vx))\Delta$ and $b = f_{y}(\vx) - f_{y'}(\vx)$ for simplicity, \emph{i.e.}, $M^\top \vdelta < b$. 
To bound the attribution differences after the sample is perturbed by noise $\vdelta$ in $\ell_2$-norm ball, \emph{i.e.}, $\Vert\vdelta\Vert_2\leq\varepsilon$, the upper bound can be formulated by rewriting the optimization problem (\ref{eq:opt_cert_ar}) as the optimal value of the following quadratic programing with concave objective function and a system of linear constraints for all labels different from $y$,
\begin{equation}\label{eqn:opt_l2_label}
      \max_{\vdelta}~ \vdelta^\top P\vdelta \quad \text{s.t. }\Vert\vdelta\Vert_2\leq\varepsilon \text{ and } M^\top \vdelta < b.
\end{equation}
To simplify the computation, in this work, we only consider the second best label $y'$, \emph{i.e.}, $y'=\argmax_{k\in\left\{1,\ldots, c\right\}\backslash y}f_k(\vx)$. In such case, the constraint $M^\top\vdelta < b$ defines a half-space. Recall that Theorem~\ref{thm:l2_unlabel} states that the upper bound in $\ell_2$-norm case without label constraint is $\xi_{max}\varepsilon$, and we noticed that this bound is achieved at two opposite vectors $\vdelta^\ast = \varepsilon\bm{v}_{max}$ or $\vdelta^\ast = -\varepsilon\bm{v}_{max}$. Thus, at least one of these two vectors lies in the half-space defined by the linear constraint (see point $A$ and $B$ in Fig.~\ref{fig:illustrate}. Therefore, the upper bound provided in Theorem~\ref{thm:l2_unlabel} is also achieved even if the label constraint is added, \emph{i.e.}, the optimal value of optimization problem (\ref{eqn:opt_l2_label}) is also $\xi_{max}\varepsilon$.

The bound of $\ell_p$-norm constrained case can be derived similarly. The less tight upper bound is still achievable at $d^{\frac{1}{2}-\frac{1}{p}}\xi_{max}\varepsilon$ as in Corollary~\ref{cor:linf_unlabel}. Generalizing the $\ell_\infty$-norm constrained upper bound using Eq.~\ref{eqn:linf_sum_ub} is simpler. According to Theorem~\ref{thm:linf_unlabel}, there are $2^d$ different optimal solutions that achieve the optimum and they are the corners of the $\ell_\infty$-norm box. As long as the feasible region is non-empty, there exists at least one corner of the box lying inside the feasible region, and the optimum value is achieved. %
Similarly, given a $d$-dimensional $\ell_\infty$-norm box, and $k$ label constraints that each separates the entire space into two half-spaces, if at least one corner of the box lies within the feasible region, the optimum value is attainable. 

\subsection{Upper bound based on cosine distance}\label{sec:cert_cos}
In the previous parts of this section, we discussed several practical upper bounds based on Euclidean distance. It is noticed that Wang \& Kong \cite{wang2022exploiting} claims that cosine similarity ($D_s$) is a better metric to measure the difference of attributions as it emphasizes the relative importance among different features rather than the absolute magnitude of each individual feature. Our method can be trivially extended to the scenarios using cosine distance ($D_c = 1-D_s(g^y(\vx+\vdelta), g^y(\vx))$) as the dissimilarity function $D$ defined in the formulation (\ref{eq:opt_cert_ar}) with simple modifications.

\begin{restatable}{corollary}{CertCosSim}\label{cor:cert_cos_sim}
  Given a twice-differentiable classifier $f: \mathbb{R}^d\rightarrow\mathbb{R}^k$ and  its attribution $g^y$ on label $y$,
  for all perturbations $\Vert\vdelta\Vert_p\leq\varepsilon$, if the Euclidean distance of $g^y(\vx+\vdelta)$ and $g^y(\vx)$ is upper bounded by $T(\varepsilon; \vx)$, and $0\leq T(\varepsilon;\vx)\leq\Vert g^y(\vx)\Vert_2$, then their cosine distance ($D_c$) is upper bounded by
  \begin{equation}
    D_c(g^y(\vx+\vdelta), g^y(\vx))\leq 1 - \sqrt{1-\frac{T(\varepsilon;\vx)^2}{\Vert g^y(\vx)\Vert^2_2}}.\label{eqn:ub_cos}
  \end{equation}
\end{restatable}
This upper bound is valid when the assumption that $0\leq T(\varepsilon;\vx)\leq\Vert g^y(\vx)\Vert_2$ is satisfied, \emph{i.e.}, the variation of attribution distance is smaller than the original attribution. As shown in Fig.~\ref{fig:illustrate}, the angle between original and perturbed attributions is bounded by $\theta$ computed from the corollary. %

\section{Experimental results}\label{sec:exp}
\begin{table*}
  \centering
  \caption{Evaluation of upper bound without the label constraint.}\label{tbl:eval_unlabel}
  \resizebox{.9\textwidth}{!}{%
  \begin{tabular}{lccccc|ccccc|ccccc}
    \toprule
     & \multicolumn{5}{c}{SM} & \multicolumn{5}{c}{Input*gradient} & \multicolumn{5}{c}{IG}\\
     \cmidrule(r){2-6}\cmidrule(r){7-11}\cmidrule(r){12-16}
     &$\widehat{T}_e$ &$T_e$&$T'_e$&$\widehat{T}_c$ &$T_c$&$\widehat{T}_e$ &$T_e$&$T'_e$&$\widehat{T}_c$ &$T_c$&$\widehat{T}_e$ &$T_e$&$T'_e$&$\widehat{T}_c$ &$T_c$\\
    \midrule
    $\ell_2$&0.09&0.31&0.34&6.88&7.41  &0.07&0.46&0.46&0.51&2.60  &0.02&0.17&0.17&1.80&3.84\\
    $\ell_\infty$&0.41&0.85&-&21.87&27.09 &0.07&0.69&-&7.03&50.59&0.25&0.52&-&23.24&35.00\\
    \bottomrule
  \end{tabular}
  }
\end{table*}
In this section, we evaluate the effectiveness of proposed upper bounds by numerical experiments under both $\ell_2$ and $\ell_\infty$-norms. %
In the following results, we compute the theoretical upper bounds for adversarial robust models and attributional robust models, including \emph{Adversarial Training} (AT) \cite{madry2018towards}, IG-NORM \cite{chen2019robust}, \emph{Adversarial Attributional Training} with robust training loss (AdvAAT) \cite{ivankay2020far}, \emph{Attributional Robustness Training} (ART) \cite{singh2020attributional}, TRADES \cite{zhang2019theoretically} and \emph{Integrated Gradients Regularizer} (IGR) \cite{wang2022exploiting}. We follow previous attribution robustness studies to use ResNet-18 to evaluate CIFAR-10 \cite{krizhevsky2009learning}, and use the neural network with four convolutional layers followed by three fully-connected layers to evaluate MNIST \cite{lecun2010mnist} and Fashion-MNIST \cite{xiao2017fashion}. The proposed upper bounds are also scalable to datasets with larger images, but due to the scalability problem of existing attribution defense methods, we provide evaluations of Flower \cite{nilsback2006visual} and ImageNet \cite{deng2009imagenet} on attributional non-robust models in Appendix~\ref{append:eval_larger}.

For each selected model, the theoretical upper bounds for both Euclidean distance and cosine distance are computed. We convert the cosine values to degrees for easier comparison. The theoretical bounds are compared with corresponding distance between original sample and attacked sample to verify the effectiveness of the bounds. We denote the theoretical upper bounds for Euclidean distance and cosine distance as $T_e$ and $T_c$, respectively. The $\ell_2$ PGD-20 attack \cite{madry2018towards} is implemented for $\ell_2$-norm bounded case. The 200-step IFIA with the top-$k$ intersection as dissimilarity function \cite{ghorbani2019interpretation} is implemented for $\ell_\infty$-norm bounded case, where $k$ is 100 for MNIST and Fashion-MNIST and 1000 for CIFAR-10. Each sample is attacked 20 times and the mean distance is computed. The sample mean Euclidean and cosine distances of the entire dataset under corresponding attacks are denoted by $\widehat{T}_e$ and $\widehat{T}_c$, respectively. All the experiments are implemented on NVIDIA GeForce RTX 3090 \footnote{Source code will be provided later.}.

In addition, we also provide a generalization of the proposed bounds based on the generalization of Theorem~\ref{thm:l2_unlabel} (Appendix~\ref{append:general_thm}) that adaptively multiply a scalar $c$ for an given input $\vx$ in case that the weak assumption is violated in rare cases. Explicitly, the adaptive value of $c$ for $i$-th sample is given as follows (details in Appendix~\ref{append:find_k})
\begin{equation}
  c^{(i)} = \max\left\{1, \frac{\Vert g^y(\vx^{(i)} + \varepsilon\bm{v}^{(i)}_{max})-g^y(\vx^{(i)})\Vert_2 }{\xi^{(i)}_{max}\varepsilon}\right\}. \label{eqn:modified_k}
\end{equation}

\subsection{Evaluation of upper bounds without the label constraint}%
We first evaluate the upper bounds without label constraints, which can be applied to any gradient-based attribution method. Here three methods are evaluated, saliency map, input*gradient and integrated gradients. The upper bounds computed from TRADES+IGR on CIFAR-10 are presented in Table~\ref{tbl:eval_unlabel}, and results from other models are deferred to Appendix~\ref{append:exp_unlabel}. We use the unlabelled upper bound introduced in Theorem~\ref{thm:l2_unlabel} and \ref{thm:linf_unlabel} to compute $T_e=\xi_{max}\varepsilon$ and extend it to $T_c$ using Eq.~\ref{eqn:ub_cos}. The perturbation size is chosen to be 0.1 for $\ell_2$ and 0.25 for $\ell_\infty$. The generalized upper bound $T'_e=c\xi_{max}\varepsilon$ (Eq.~\ref{eqn:modified_k}) is also provided for $\ell_2$ case, and is not necessary for $\ell_\infty$ case. As observed in the table, the proposed upper bounds are valid for different attribution methods and both Euclidean and cosine distances are well-bounded. More precisely, none of the $\ell_\infty$ perturbed attributions is outside the theoretical bound and none of the $\ell_2$ perturbed attributions is outside the generalized bound.

\subsection{Evaluation of upper bounds under \texorpdfstring{$\ell_2$}{L2}-norm constraint and label constraint}%
To evaluate the upper bound of attribution deviations after samples being attacked by $\ell_2$-norm constrained perturbations,
we use the method provided in Sec.~\ref{sec:cert_labeled} and \ref{sec:cert_cos} to obtain the theoretical upper bounds for both Euclidean distance ($T_e=\xi_{max}\varepsilon$) and cosine distance ($T_c$ in Eq.~\ref{eqn:ub_cos}). Integrated gradients (IG) is chosen here as the theoretical bound with the label constraint is based on the axiom of completeness. Besides, as discussed in Sec.~\ref{sec:l2_nonlabel}, the percentages of attacked attribution outside $T_e$ are provided, and the generalized bound is also calculated and denoted by $T'_e=c\xi_{max}\varepsilon$. Since $T'_e$ bounds all the attacked attributions, \emph{i.e.}, 100\% for all the models, we do not report the percentages in the table.  %
From Table~\ref{tbl:eval_label_l2}, we observe the following results. (i) The percentages are low, which supports our assumption in Sec.~\ref{sec:l2_nonlabel} that $g^y$ is locally linear. (ii) The computed values for both Euclidean ($T'_e$) and cosine distance ($T_c$) successfully bound the attribution differences above for every dataset and every model. In addition, we also show the minimum Euclidean gaps between samples and bounds in Section~\ref{append:tightness} to illustrate that the tightness of the proposed bounds.

\begin{table}
  \caption{Evaluation of upper bounds under $\ell_2$-norm constraint and label constraint. The numbers in the brackets indicate the percentages that attacked attribution is outside the $T_e$.}\label{tbl:eval_label_l2}
  \centering
  \resizebox{.47\textwidth}{!}{%
    \begin{tabular}{lccccc}
      \toprule
      Model &$\widehat{T}_e$ &$T_e$ &$T'_e$ &$\widehat{T}_c$($\deg$) &$T_c$($\deg$)\\
      \midrule
      $\varepsilon=0.05$ &\multicolumn{5}{c}{MNIST}\\
      \midrule
      AT &0.0685&0.1537 [2.25\%]&0.1596 &3.6935&7.0951\\
      IG-NORM &0.1158&0.2888 [2.00\%]&0.2967&3.3174&7.2615\\
      ART &0.0626&0.3591 [6.00\%]&0.3702&2.3923&6.8657\\
      AdvAAT &0.0876&0.3269 [6.20\%]&0.3404&1.9034&6.8992\\
      TRADES &0.1620&0.5060 [1.68\%]&0.5271&2.8374&6.9988\\
      TRADES+IGR &0.1784&0.4964 [1.32\%]&0.5145&2.9075&6.9779\\
      \midrule
      $\varepsilon=0.05$ &\multicolumn{5}{c}{Fashion-MNIST}\\
      \midrule
      AT & 0.0659 & 0.0700 [2.19\%] &0.0869& 10.1442& 12.9577  \\ 
      IG-NORM & 0.1181 & 0.1789 [0.00\%] &0.1789& 6.6002 & 8.8043 \\ 
      AdvAAT & 0.1115 & 0.1735 [6.20\%] &0.1858& 5.8544 & 9.3692 \\ 
      ART &0.0940&0.1387 [0.94\%]&0.1411&5.4507& 9.8234\\
      TRADES & 0.0626 & 0.0963 [4.16\%] &0.1184& 8.3521 & 12.0991 \\ 
      TRADES+IGR & 0.0403& 0.0453 [1.91\%] &0.0507& 7.7302& 8.8411  \\ 
      \midrule
      $\varepsilon=0.1$ &\multicolumn{5}{c}{CIFAR-10}\\
      \midrule
      AT&0.0392&0.2532 [0.09\%]&0.2533&2.7335&4.7724\\
      IG-NORM &0.0149&0.1582 [0.42\%]&0.1621&1.6505&4.3711\\
      AdvAAT&0.0374&0.2386 [0.06\%]&0.2386&0.2847&3.8202\\
      ART &0.0733&0.2278 [0.00\%]&0.2278&0.5918&4.2123\\
      TRADES&0.0264&0.1734 [0.16\%]&0.1734&1.9084&3.8686\\
      TRADES+IGR&0.0240&0.1692 [0.09\%]&0.1692&1.8011&3.8384\\
      \bottomrule
    \end{tabular}
  }
\end{table}

\begin{figure*}
  \centering
  \begin{subfigure}{.28\textwidth}
      \centering
      \includegraphics[width=\textwidth]{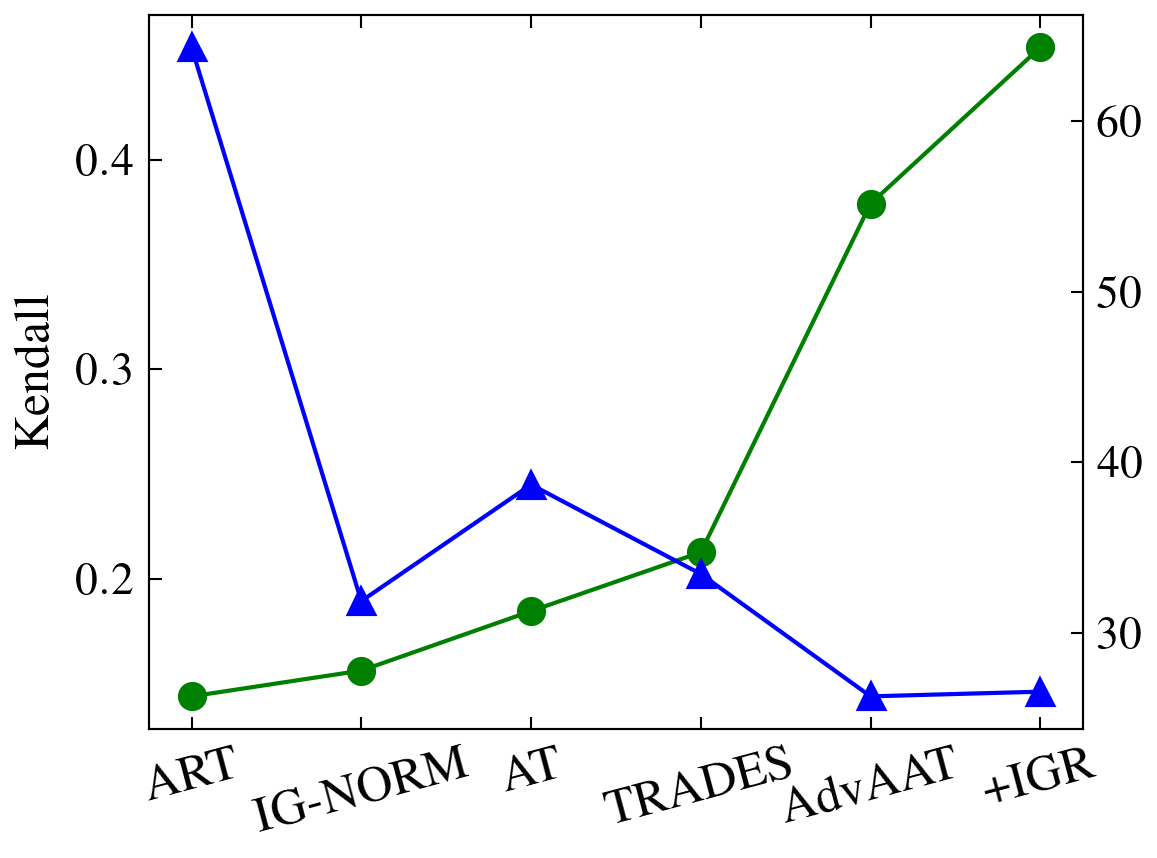}
      \caption{MNIST}
  \end{subfigure}
  \hfill
  \begin{subfigure}{.28\textwidth}
      \centering
      \includegraphics[width=\textwidth]{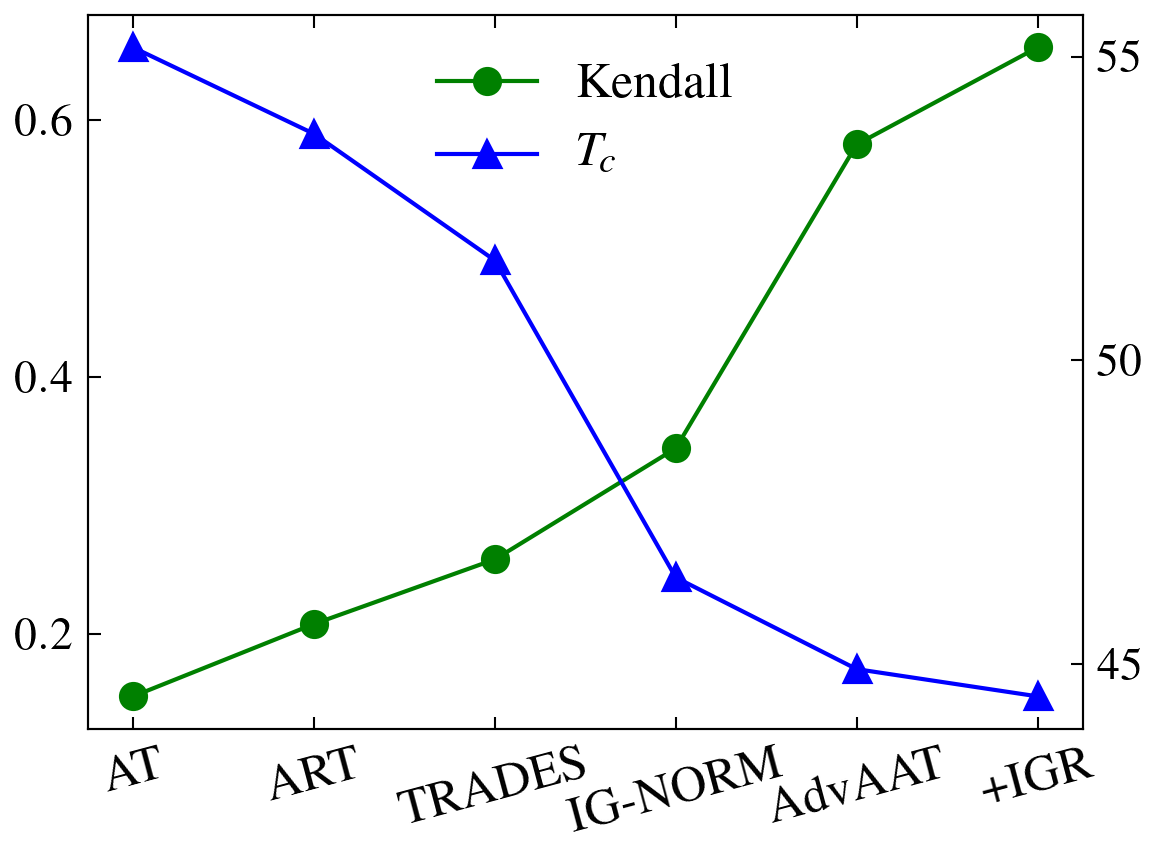}
      \caption{Fashion-MNIST}
  \end{subfigure}
  \hfill
  \begin{subfigure}{.28\textwidth}
      \centering
      \includegraphics[width=\textwidth]{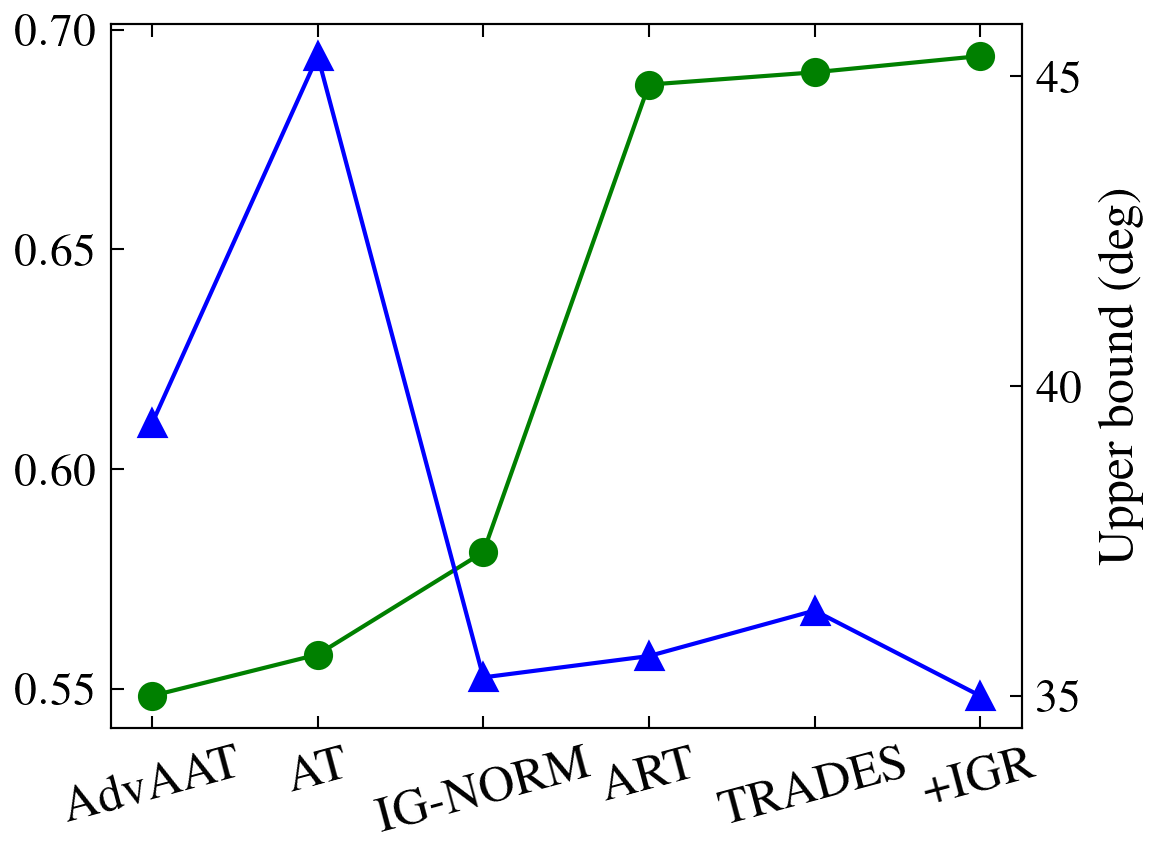}
      \caption{CIFAR-10}
  \end{subfigure}
  \caption{Comparison between Kendall's rank correlation and the theoretical bound of cosine distance. For clear comparison, we convert the cosine distance to angles in degrees.}\label{fig:kend_cos}
\end{figure*}

\subsection{Evaluation of upper bounds under \texorpdfstring{$\ell_\infty$}{Linf}-norm constraint and label constraint}%

\begin{table}
  \caption{Evaluation of upper bounds $\ell_\infty$-norm constraint and label constraint.}\label{tbl:eval_label_linf}
  \centering
  \resizebox{.47\textwidth}{!}{%
  \begin{tabular}{lccccc}
    \toprule
    Model &Kendall &$\widehat{T}_e$ &$T_e$  &$\widehat{T}_c$($\deg$) &$T_c$($\deg$)\\
    \midrule
    $\varepsilon=0.05$ &\multicolumn{5}{c}{MNIST}\\
    \midrule
    AT &0.1846&0.3461&0.7752  &15.3616&38.7024\\
    IG-NORM &0.1562&0.6836&1.2046 &16.4506&31.8791\\
    AdvAAT &0.3791&1.6269&2.1992 &11.3173&26.3000\\
    ART &0.1439&1.4193&2.8218&12.8025&64.3115\\
    TRADES &0.2127&1.1779&2.2216 &15.9881&33.4681\\
    TRADES+IGR &0.4537&1.2991&2.0386 &17.5923&26.5748\\
    \midrule
    $\varepsilon=0.05$ &\multicolumn{5}{c}{Fashion-MNIST}\\
    \midrule
    AT &0.1516& 0.0990& 0.2802   & 18.9720   & 55.1501\\
    IG-NORM &0.3446& 0.2384& 0.8819  & 12.6023    & 46.4270\\
    AdvAAT &0.5810& 0.1938& 0.9206  & 9.4499 & 44.9184\\
    ART &0.2079&0.1660&0.7215&9.6582&53.7281 \\
    TRADES &0.2582& 0.1042& 0.4536  & 13.7010 & 51.6448\\
    TRADES+IGR &0.6565& 0.0722& 0.2526  & 15.0947 & 44.4703\\
    \midrule
    $\varepsilon=0.1$ &\multicolumn{5}{c}{CIFAR-10}\\
    \midrule
    AT&0.5578&0.4058&0.7649&26.6195&45.3223\\
    IG-NORM &0.5811&0.1997&0.4783&21.6311&35.2981\\
    AdvAAT &0.5484&0.2293&0.5211&28.7342&39.3981\\
    ART &0.6875&0.3128&0.6734&31.0090&35.6422\\
    TRADES&0.6903&0.2322&0.5001&22.9779&36.3759\\
    TRADES+IGR&0.6940&0.2474&0.5236&23.2356&35.0009\\
    \bottomrule
  \end{tabular}
  }
\end{table}

\begin{table*}
  \centering
  \caption{Evaluation of tightness of the bounds in Euclidean distance for $\ell_2$ and $\ell_\infty$ cases.}\label{apxtbl:l2_gap}
  \resizebox{.8\textwidth}{!}{%
  \begin{tabular}{lccc|ccc|ccc}
    \toprule
    &\multicolumn{3}{c}{MNIST}&\multicolumn{3}{c}{Fashion-MNIST}&\multicolumn{3}{c}{CIFAR-10}\\
    \cmidrule(r){2-4}\cmidrule(r){5-7}\cmidrule(r){8-10}
    $\ell_2 (\varepsilon=)$ &0.05&0.1&0.2&0.05&0.1&0.2&0.1&0.2&0.3\\
    \midrule
    AT          & 0.0260 & 0.0320 & 0.0140 & 0.0078 & 0.0130 & 0.1207 & 0.0004 & 0.0045 & 0.0134 \\
    IG-NORM     & 0.0391 & 0.0597 & 0.0291 & 0.0121 & 0.0171 & 0.0742 & 0.0016 & 0.0210 & 0.0448 \\
    AdvAAT      & 0.0290 & 0.0465 & 0.0294 & 0.0103 & 0.0167 & 0.0183 & 0.0037 & 0.0090 & 0.0181 \\
    ART         & 0.0178 & 0.0115 & 0.0182 & 0.0029 & 0.0239 & 0.0011 & 0.0019 & 0.0031 & 0.0141 \\
    TRADES      & 0.0014 & 0.0032 & 0.0104 & 0.0037 & 0.0082 & 0.0723 & 0.0028 & 0.0048 & 0.0147 \\
    TRADES+IGR  & 0.0010 & 0.0038 & 0.0041 & 0.0064 & 0.0134 & 0.0385 & 0.0016 & 0.0100 & 0.0126 \\
    \midrule
    \midrule
    $\ell_\infty (\varepsilon=)$ &0.01&0.03&0.05&0.01&0.03&0.05&4/255&8/255&0.1\\
    \midrule
    AT          & 0.0021 & 0.0035 & 0.0117 & 0.0004 & 0.0352 & 0.0708 & 0.0025 & 0.0053 & 0.1381 \\
    IG-NORM     & 0.0001 & 0.0062 & 0.0105 & 0.0010 & 0.0590 & 0.1069 & 0.0026 & 0.0145 & 0.0003 \\
    AdvAAT      & 0.0004 & 0.0223 & 0.0901 & 0.0136 & 0.0847 & 0.1665 & 0.0448 & 0.1513 & 0.2078 \\
    ART         & 0.0082 & 0.0112 & 0.0233 & 0.0424 & 0.1049 & 0.1467 & 0.0118 & 0.0412 & 0.0870 \\
    TRADES      & 0.0001 & 0.0046 & 0.0026 & 0.0014 & 0.0337 & 0.0634 & 0.0068 & 0.0043 & 0.0530 \\
    TRADES+IGR  & 0.0016 & 0.0143 & 0.1389 & 0.0014 & 0.0375 & 0.0682 & 0.0016 & 0.0090 & 0.0197 \\
    \bottomrule
  \end{tabular}
  }
\end{table*}
In this subsection, the upper bounds of attribution deviations under $\ell_\infty$ attacks are presented. Instead of the much looser bound derived from $\ell_p$-norm relaxation (Corollary~\ref{cor:linf_unlabel}), the upper bound is computed from $T_e=\varepsilon\sqrt{\sum \vert P_{ij}\vert}$ as introduced in Sec.~\ref{sec:cert_labeled} and \ref{sec:cert_cos}
. Moreover, the empirical attribution robustness is also provided using Kendall's rank correlation \cite{kendall1948rank} for comparison of theoretical and empirical protection. It should be emphasized that all the previous attribution robustness studies are based on $\ell_\infty$-norm constraints. Kendall's rank correlation, which is non-differentiable, is used as one of the indexes to measure the difference between the original attributions and the attributions attacked by IFIA under $\ell_\infty$-norm constraints. From the results in Table~\ref{tbl:eval_label_linf}, we see that the computed theoretical upper bounds are valid to measure the worst-case attribution deviations. For each dataset and each model, the sample mean of attribution distances is strictly smaller than the theoretical distance. Because of the relaxation, all attacked attributions are bounded by $T_e$. There is no outlier, and there is no need to use the generalized bound $T'_e$. Moreover, the results also show that for a model with a larger Kendall's rank correlation, the theoretical cosine distance upper bound is more likely to be smaller (see Fig.~\ref{fig:kend_cos}), which means that the model is more difficult to be attacked. This also confirms that cosine similarity is positively correlated to Kendall's rank correlation as proposed in Wang \& Kong \cite{wang2022exploiting}.

\subsection{Evaluation of the tightness of bounds}\label{append:tightness}
In addition, we further report the minimum Euclidean gaps between samples and theoretical bounds in Table \ref{apxtbl:l2_gap} to measure the tightness of the provided bounds, which is defined as
\begin{equation}
  r = \min_{0\leq i\leq n} T^{(i)}_e - \widehat{T}_e^{(i)}
\end{equation}
Note that the superscript $(i)$ represents the $i$-th sample and $T^{(i)}_e$ is replaced by $T'^{(i)}_e$ in $\ell_2$-norm cases. $r$ is a straightforward measurement of the theoretical bound. We notice that although the mean of theoretical bounds sometimes are multiple times larger than the sample mean distance, the tightest bound can be only $10^{-4}$ greater than the sample distance. We also observe that the values of $r$ are all positive, which also indicates that there is no perturbed attribution that violates our theoretical bounds.

We also provide the visualizations of the distribution of the gap between theoretical bounds and attribution differences from real data in Appendix~\ref{append:additional_exp}. The values are directly computed using $T^{(i)}_e - \widehat{T}_e^{(i)}$. As we can observe from the figures, all values are positive, which verifies the validity of our bounds, and most of the values are lying close to $0$, which shows the tightness of the bounds.

\section{Conclusion}\label{sec:conclusion}
In this paper, an upper bound that measures the worst-case attribution deviations is proposed. The bound is formulated as the optimum value of a constrained optimization problem. The optimization problem is constrained on the size of perturbation and the unchanged classification label after being perturbed. For each of the two metrics of the attribution difference, Euclidean and cosine distances, the problem is solved based on the first-order Taylor series and the estimation of attribution gradients and the solution bounds the attribution deviations under both $\ell_2$ and $\ell_\infty$-norm attacks. Experimental results validate the effectiveness of the proposed bounds.

\section*{Acknowledgments and Disclosure of Funding}
    This work is partially supported by the Ministry of Education, Singapore through Academic Research Fund Tier 1, RG73/21.
  
{\small
\bibliographystyle{ieee_fullname}
\bibliography{cvpr_2023}
}

\newpage
\appendix
\onecolumn
\section{Proofs}

\subsection{Proof of Corollary~\ref{cor:linf_unlabel}}\label{proof:cor_linf_unlabel}
Before we prove Corollary~\ref{cor:linf_unlabel}, We first introduce the following lemma.
\begin{lemma}\label{lemma:norm}
  For $0< q < p$, the following inequality holds:
  \begin{equation}
      \Vert \vx\Vert_q \leq d^{\frac{1}{q}-\frac{1}{p}}\Vert \vx\Vert_p
  \end{equation}
  where $\vx\in\mathbb{R}^d$.
\end{lemma}
\begin{proof}
  Consider $\bm{u}, \bm{v}\in\mathbb{R}^d$, using the H\"{o}lder's Inequality that for $m, n$ satisfying $\frac{1}{m}+\frac{1}{n}=1$, 
  \begin{equation}
      \sum_i \vert u_i\vert\vert v_i\vert\leq\left(\sum_i\vert u_i\vert^m\right)^{\frac{1}{m}}\left(\sum_i\vert v_i\vert^n\right)^{\frac{1}{n}}.
  \end{equation}
  If we take $\vert u_i\vert=\vert x_i\vert^q$, $v_i=1$, $m=\frac{p}{q}$ and $n=\frac{p}{p-q}$, we get
  \begin{equation}
      \sum_i\vert x_i\vert^q \leq\left(\sum_i \vert x_i\vert^p\right)^{\frac{q}{p}}d^{\frac{p-q}{p}}
  \end{equation}
  By taking the power of $\frac{1}{q}$ on both sides, we have
  \begin{equation}
      \left(\sum_i\vert x_i\vert^q\right)^\frac{1}{q} \leq \left(\sum_i\vert x_i\vert^p\right)^\frac{1}{p}d^{\frac{1}{q}-\frac{1}{p}}
  \end{equation} 
  which concludes the proof.
\end{proof}

\linfunlabel*
\begin{proof}
  Using Lemma~\ref{lemma:norm}, we have $\Vert\vdelta\Vert_2\leq d^{\frac{1}{2}-\frac{1}{p}}\Vert\vdelta\Vert_p$. Similar to the proof of Theorem~\ref{thm:l2_unlabel},
  \begin{equation}
    \Vert g^y(\vx+\vdelta) - g^y(\vx) \Vert_2^2 \leq \lambda_{max}\Vert\vdelta\Vert_2^2\leq \lambda_{max}\left(d^{\frac{1}{2}-\frac{1}{p}}\Vert\vdelta\Vert_p\right)^2\leq \lambda_{max}\left(d^{\frac{1}{2}-\frac{1}{p}}\varepsilon\right)^2
  \end{equation}
  Therefore,
  \begin{equation}
    \Vert g^y(\vx+\vdelta) - g^y(\vx) \Vert_2\leq d^{\frac{1}{2}-\frac{1}{p}}\xi_{max}\varepsilon
  \end{equation}
\end{proof}

\subsection{Proof of Theorem~\ref{thm:linf_unlabel}}\label{proof:thm_linf_unlabel}
\thmlinfunlabel*
\begin{proof}
  Recall that under the local linearity assumption, 
  \begin{equation}
    \Vert g^y(\vx + \vdelta) - g^y(\vx)\Vert_2^2\leq \vdelta^\top P \vdelta = \sum_{i,j}P_{ij}\delta_i\delta_j.
  \end{equation}
  Since $P_{ij}\leq\vert P_{ij}\vert$ and $\delta_i\delta_j\leq\Vert\vdelta\Vert^2_\infty\leq\varepsilon^2$ for all $i, j$, we can easily prove the theorem that
  \begin{equation}
    \Vert g^y(\vx + \vdelta) - g^y(\vx)\Vert_2^2 \leq \varepsilon^2\sum_{i,j}\vert P_{ij}\vert.
  \end{equation}
\end{proof}

\subsection{Proof of Proposition~\ref{prop:cond_label}}\label{proof:prop_cond_label}
\propcondlabel*
\begin{proof}
  Recall that we denote the gradient-based attribution satisfying the completeness axiom of $\vx$ on target label $y$ by $g^y(\vx)$, \emph{e.g.}, integrated gradients. Similarly, we denote the attribution on a different label $y'$ by $g^{y'}(\vx)$. Given the perturbation $\vdelta$, according to the above assumption, we can write that
  \begin{equation}
      g^y(\vx + \vdelta) = g^y(\vx) + \nabla g^y(\vx)^\top \vdelta
  \end{equation}
  Similarly, the approximation of $g^{y'}(\vx + \vdelta)$ is given by:
  \begin{equation}
      g^{y'}(\vx + \vdelta) = g^{y'}(\vx) + \nabla g^{y'}(\vx)^\top \vdelta
  \end{equation}
  According to the completeness axiom, given an all one vector $\Delta = (1, \ldots, 1)^\top$, we have
  \begin{equation}
      \Delta^\top g^{y}(\vx) = f_y(\vx).
  \end{equation}
  Consider the perturbation $\vdelta$, if $\vdelta$ does not change the label of $\vx$ from $y$ to $y'$, then $f_{y'}(\vx + \vdelta) < f_y(\vx+\vdelta)$, \emph{i.e.},
  \begin{equation}
      \Delta^\top g^{y'}(\vx+\vdelta) < \Delta^\top g^{y}(\vx+\vdelta),
  \end{equation}
  which gives
  \begin{equation}
      \Delta^\top g^{y'}(\vx) + \Delta^\top \nabla g^{y'}(\vx)^\top \vdelta < \Delta^\top g^{y}(\vx) + \Delta^\top \nabla g^{y}(\vx)^\top \vdelta.
  \end{equation}
  By rearranging the above inequality, we have 
  \begin{equation}
      \left(\left(\nabla g^{y'}(\vx) - \nabla g^{y}(\vx)\right)\Delta\right)^\top \vdelta < f_{y}(\vx) - f_{y'}(\vx).
  \end{equation}
\end{proof}

\subsection{Proof of Corollary~\ref{cor:cert_cos_sim}}\label{proof:cor_cert_cos_sim}
\CertCosSim*
\begin{proof}
  The corollary can be proved using the geometric property (see Fig.~\ref{fig:illustrate}) that 
  \begin{equation}
    \sin(g^y(\vx+\vdelta), g^y(\vx)) \leq \frac{T(\varepsilon;\vx)}{\Vert g^y(\vx)\Vert_2},
  \end{equation}
  and,
  \begin{align}
    \text{cosd}(g^y(\vx+\vdelta), g^y(\vx)) &= 1-\cos(g^y(\vx+\vdelta), g^y(\vx))\\
    &= 1-\sqrt{1-\sin^2(g^y(\vx+\vdelta), g^y(\vx))}\\
    &\leq 1- \sqrt{1 - \frac{T(\varepsilon;\vx)^2}{\Vert g^y(\vx)\Vert_2^2}}
  \end{align}
\end{proof}

\section{Analysis of local linearity assumption}\label{append:eval_locallinear}
\subsection{Evaluation of local linearity assumption of attribution functions}
The theories of this work are based on the local linearity assumption that $g^y(\vx)$ is linear within $\mathcal{B}_{\varepsilon}(\vx) = \{\vx + \vdelta \vert \Vert\vdelta\Vert_p\leq \varepsilon\}$. It is worth noting that such local linearity is a valid assumption for smooth functions, which can be achieved by both adversarial and attributional robust methods. Adversarial defense methods look for locally linearity functions to reduce the impact of adversarial attacks \cite{qin2019adversarial,yang2020closer}. Similarly, attributional defense methods train for smooth gradients to defend against attribution attacks \cite{wang2020smoothed}. It is also a common practice in related literature \cite{finlay2019scaleable,guo2019mixup,simon2019first,laidlaw2020perceptual,zhang2020does} to make similar assumptions. 

Furthermore, the validity of this assumption also depends on the size of $\vdelta$. The perturbation $\vdelta$ is restricted within a small $\ell_p$ ball around $\vx$ to ensure that the perturbed images are visually indistinguishable comparing to its original counterpart. The maximum allowable size $\varepsilon$ for $\vdelta$ is relatively small compared with the intensity range of the original image. When $\vdelta$ is small, the remainder of the Taylor series of $g^y(\vx)$ is negligible and the local linearity assumption is valid. As shown in Figure~\ref{apxfig:eta}, the value of $\eta(\vx, \vdelta) = \Vert g^y(\vx)-g^y(\vx+\vdelta)-\vdelta^\top\nabla g^y(\vx)\Vert_2$ is small and negligible when $\Vert\vdelta\Vert_\infty$ is small.

\begin{figure}
    \centering
    \begin{subfigure}{.32\textwidth}
      \centering
      \includegraphics[width=\textwidth]{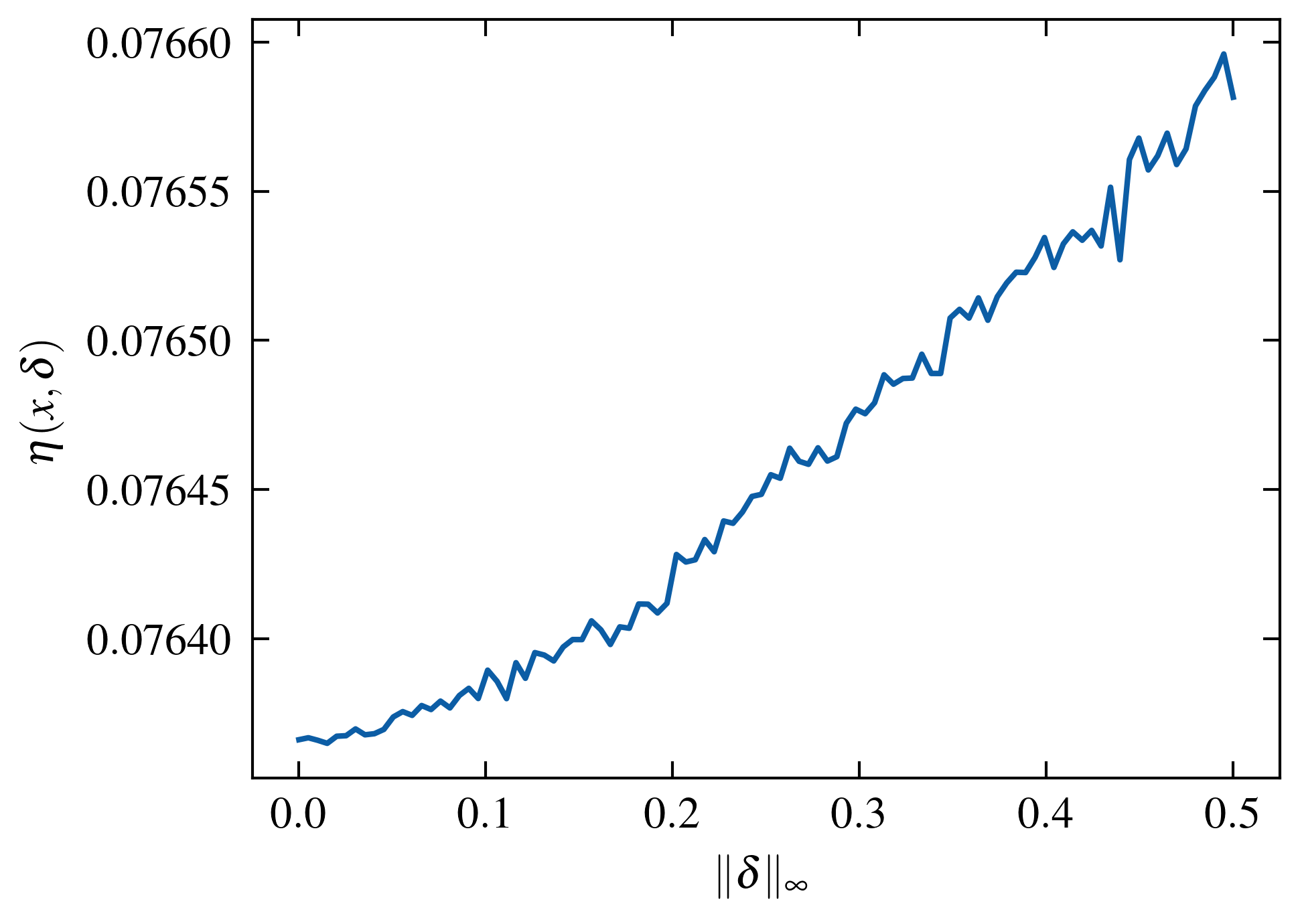}
      \caption{AT}
    \end{subfigure}
    \begin{subfigure}{.32\textwidth}
      \centering
      \includegraphics[width=\textwidth]{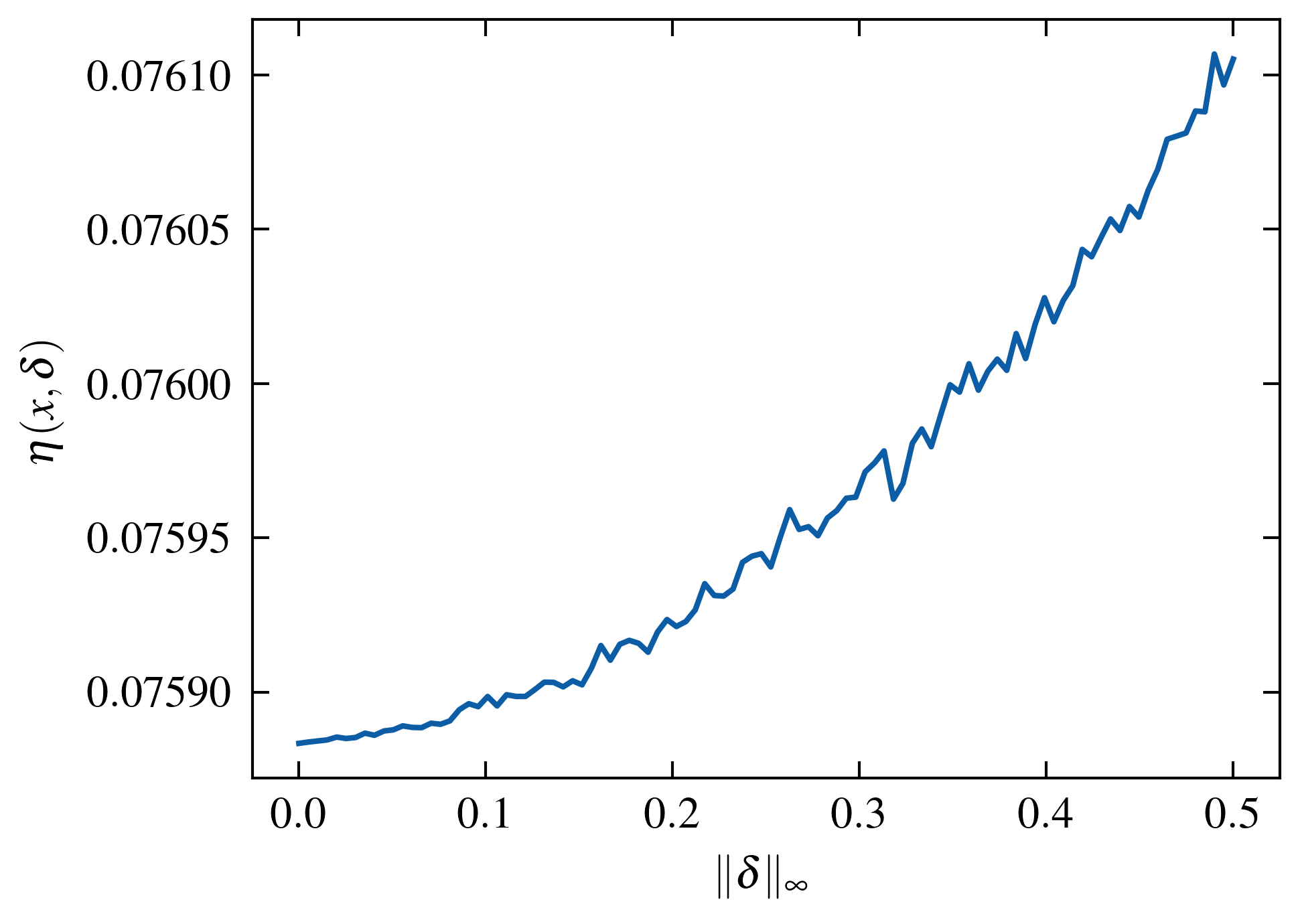}
      \caption{IG-NORM}
    \end{subfigure}
    \begin{subfigure}{.32\textwidth}
      \centering
      \includegraphics[width=\textwidth]{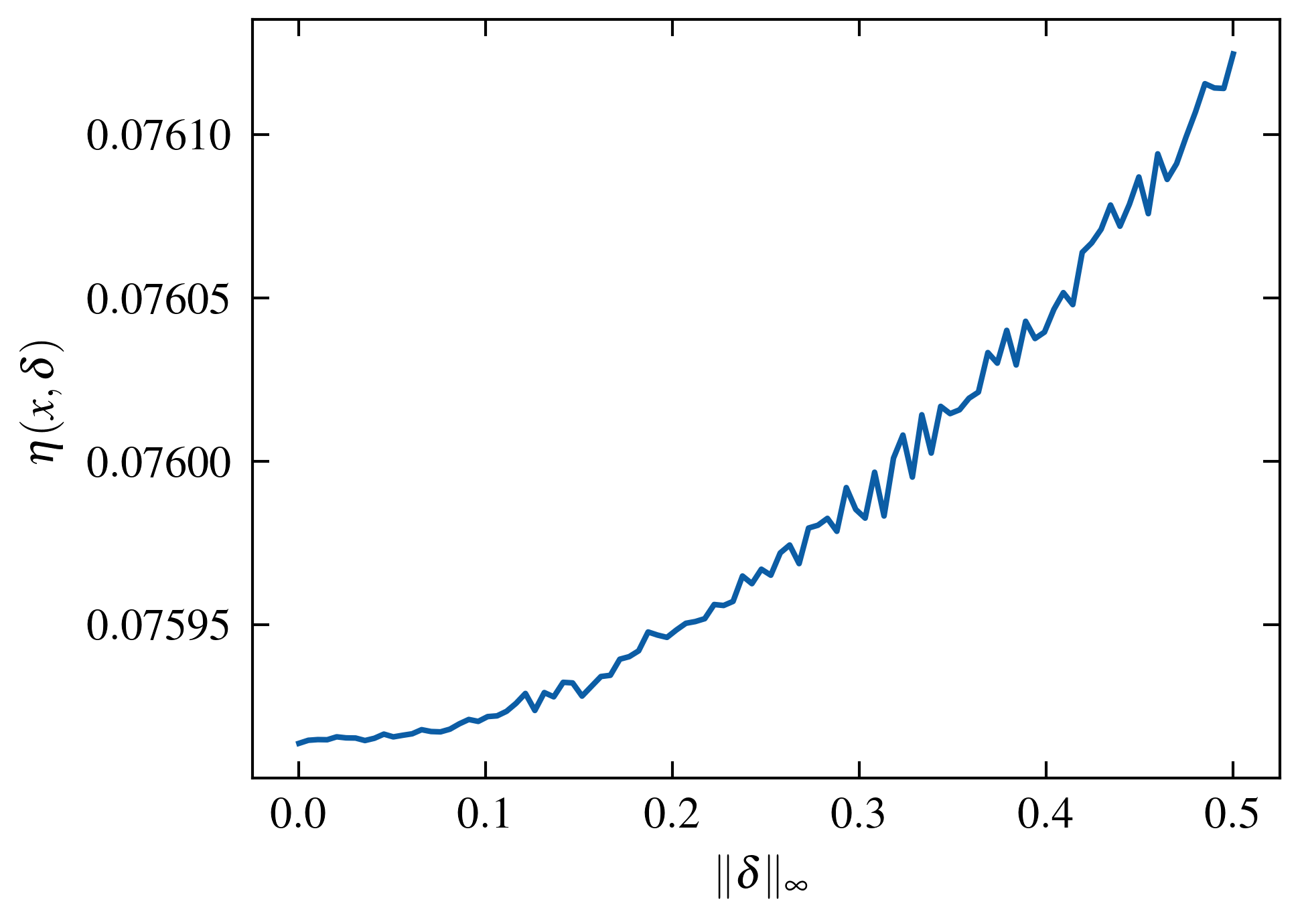}
      \caption{AdvAAT}
    \end{subfigure}
    \begin{subfigure}{.32\textwidth}
      \centering
      \includegraphics[width=\textwidth]{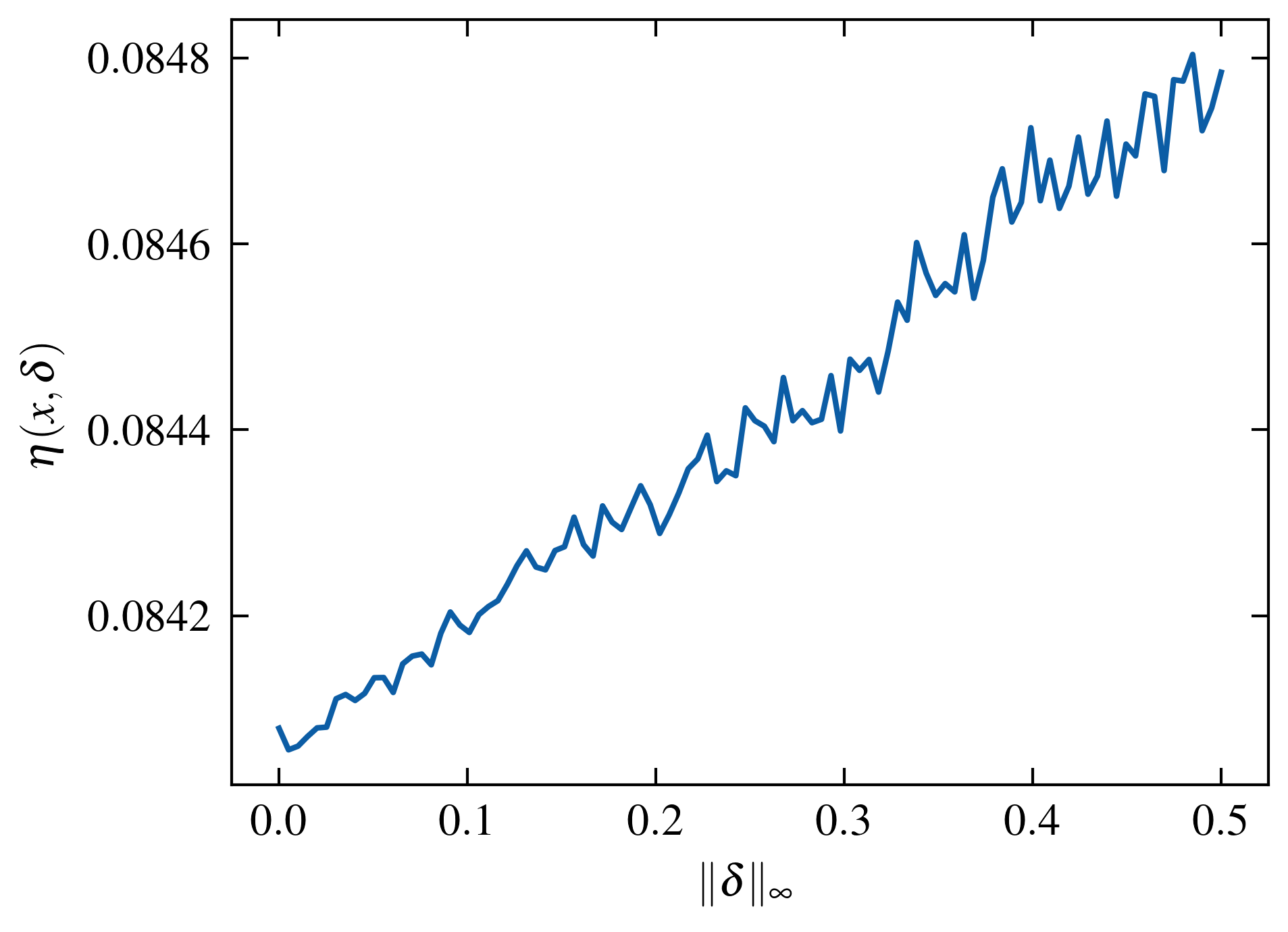}
      \caption{ART}
    \end{subfigure}
    \begin{subfigure}{.32\textwidth}
      \centering
      \includegraphics[width=\textwidth]{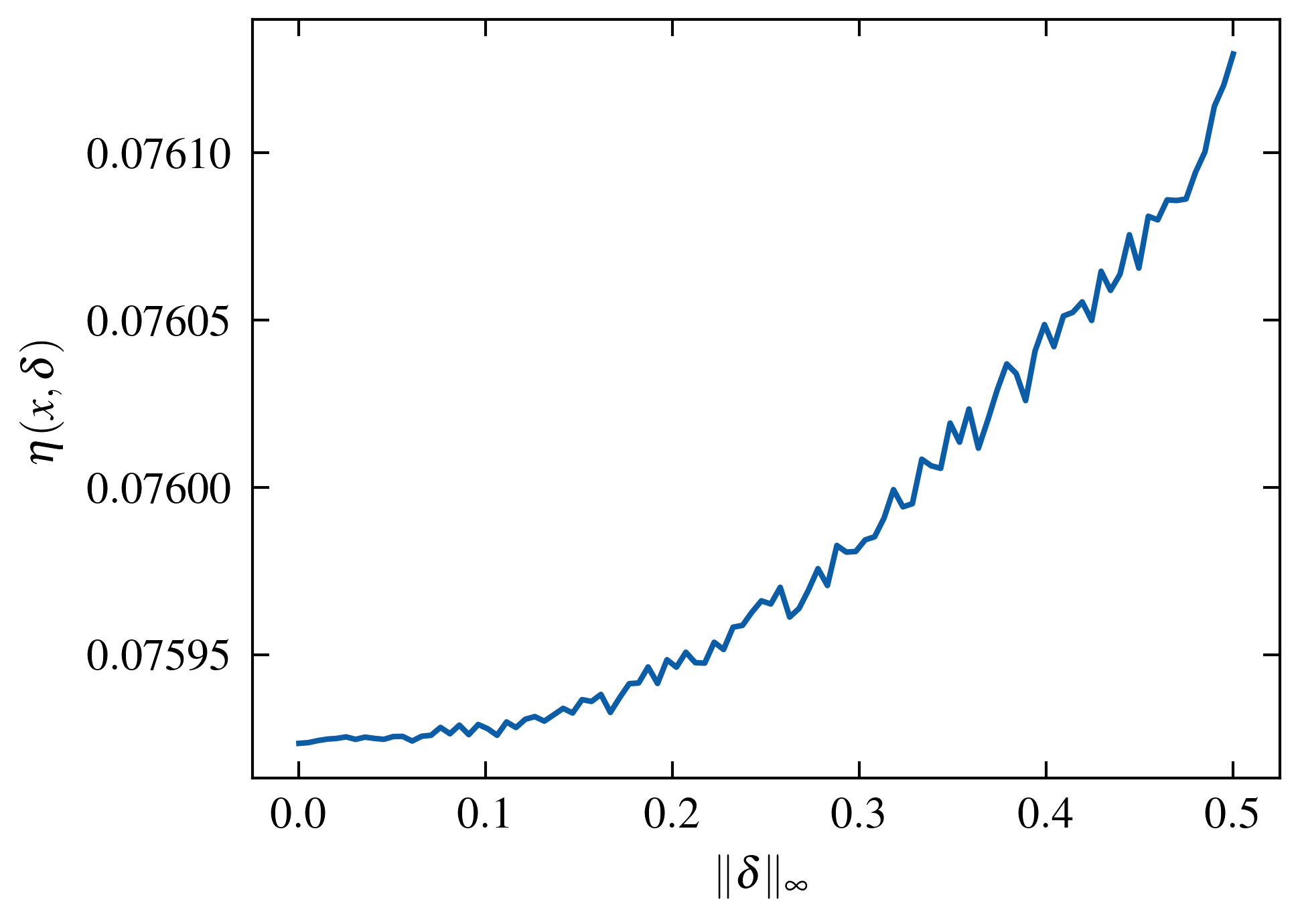}
      \caption{TRADES}
    \end{subfigure}
    \begin{subfigure}{.32\textwidth}
      \centering
      \includegraphics[width=\textwidth]{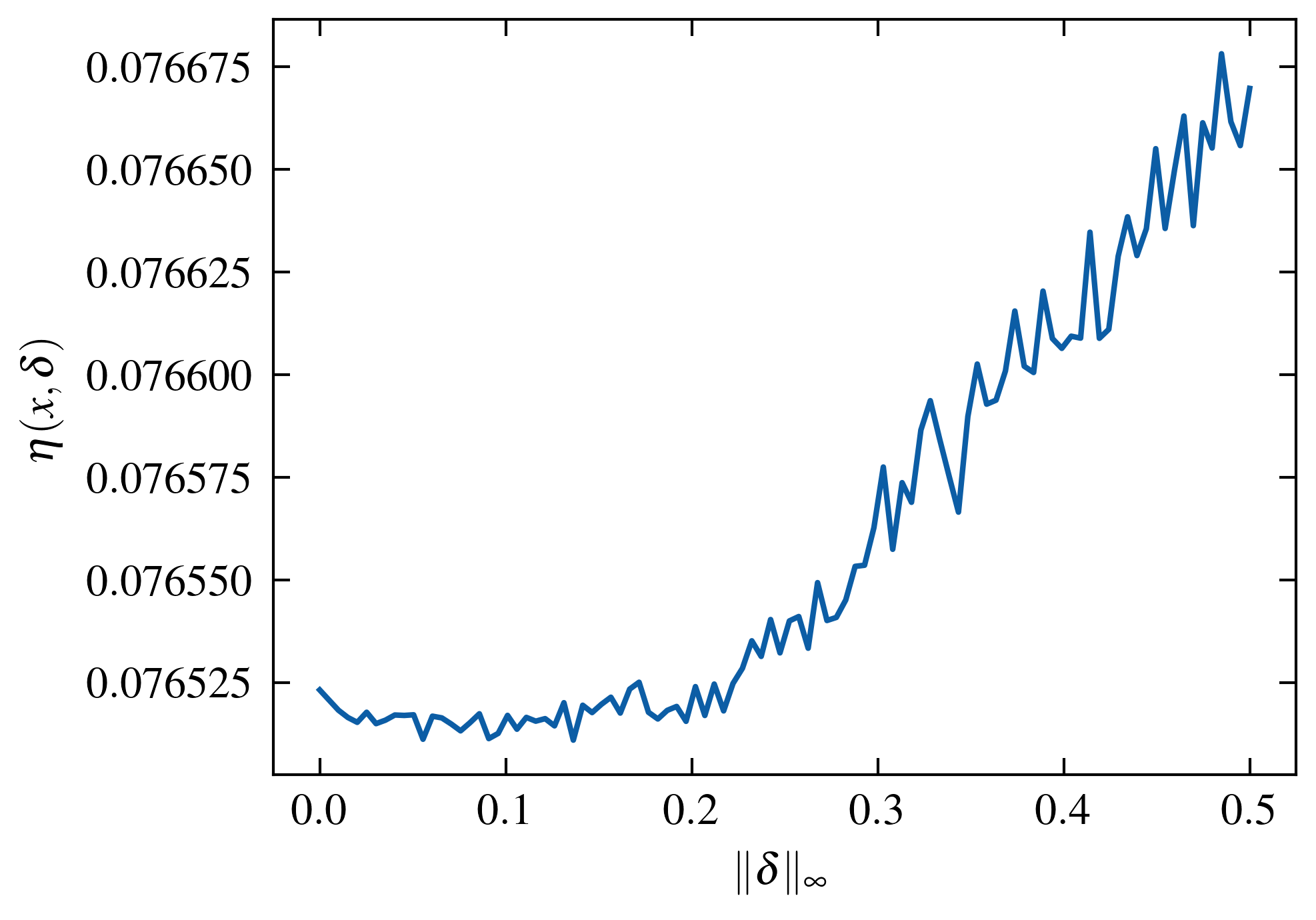}
      \caption{TRADES+IGR}
    \end{subfigure} 
    \caption{Values of $\eta$ for different $\Vert\vdelta\Vert_\infty$ computed from CIFAR-10 using integrated gradients. The magnitudes are ranging from $0.07$ to $0.09$ and are negligible comparing with the average norm of attributions which is $3.47$ on CIFAR-10.}\label{apxfig:eta}
\end{figure}  

\subsection{Generalization of Theorem~\ref{thm:l2_unlabel}}\label{append:general_thm}
\begin{theorem}
  Given a twice-differentiable classifier $f: \mathbb{R}^d\rightarrow\mathbb{R}^k$, and its attribution $g^y$ on label $y$, denote the Taylor series of $g^y(\vx+\vdelta)$ as $g^y(\vx)+\vdelta^\top\nabla g^y(\vx) + R_1(\vx)$. If $-(c-1)\vdelta^\top\nabla g^y(\vx) \preceq R_1(\vx) \preceq (c-1)\vdelta^\top\nabla g^y(\vx)$ for a constant $c\geq 1$, where $\preceq$ refers to element-wise less than or equal to, then for all perturbations $\Vert\vdelta\Vert_2\leq\varepsilon$, 
  $$\Vert g^y(\vx+\vdelta)-g^y(\vx)\Vert_2\leq c\xi_{max}\varepsilon,$$ 
  where $\xi_{max}$ is the largest singular value of $H=\nabla g^y(\vx)$.
\end{theorem}

\begin{proof}
  Based on the Taylor series of $g^y(\vx)$ and the above condition, we have
  \begin{align}
    \Vert g^y(\vx+\vdelta) - g^y(\vx) \Vert_2^2 \leq \Vert \vdelta^\top\nabla g^y(\vx) + (c-1)\vdelta^\top\nabla g^y(\vx)\Vert_2^2 &= c^2\vdelta^\top\nabla g^y(\vx)\nabla g^y(\vx)^\top\vdelta\\
    &= c^2\frac{\vdelta^\top}{\Vert\vdelta\Vert_2}P \frac{\vdelta}{\Vert\vdelta\Vert_2} \cdot \Vert\vdelta\Vert^2_2 \\
    &\leq c^2\lambda_{max}\Vert\vdelta\Vert^2_2\leq c^2\lambda_{max}\varepsilon^2\label{apxeqn:equality}
  \end{align}
  where $\lambda_{max}$ is the largest eigenvalue of $P = HH^\top = \nabla g^y(\vx)\nabla g^y(\vx)^\top$, and $\bm{v}_{max}$ is the corresponding eigenvector. The equality in Eq.~\ref{apxeqn:equality} is achieved when $\vdelta$ is $\varepsilon\bm{v}_{max}$ or $-\varepsilon\bm{v}_{max}$. Since the singular values of $H$ are equal to the square root of the eigenvalues of $P$, then,
  \begin{equation}
    \Vert g^y(\vx+\vdelta) - g^y(\vx) \Vert_2 \leq c\sqrt{\lambda_{max}}\varepsilon = c\xi_{max}\varepsilon.
  \end{equation}
\end{proof}
This is a generalized version of Theorem~\ref{thm:l2_unlabel} that is applicable for all twice-differentiable classifiers. Under local linearity assumption, $R_1(\vx)=0$, which means $c=1$, the result coincides with the original version of Theorem~\ref{thm:l2_unlabel}.

\subsection{Derivation of Eq.~(\ref{eqn:modified_k})}\label{append:find_k}
By Taylor expansion, $g^y(\vx+\vdelta) - g^y(\vx) = \vdelta^\top\nabla g^y(\vx) + R_1(\vx)$, where $R_1$ is the first order Taylor remainder. Thus, we have
\begin{equation}
  \Vert R_1(\vx)\Vert_2 \geq \Vert g^y(\vx+\vdelta) - g^y(\vx)\Vert_2 - \Vert\vdelta^\top\nabla g^y(\vx)\Vert_2
\end{equation}
Take $c=\frac{\Vert R_1(\vx)\Vert_2}{\Vert\vdelta^\top\nabla g^y(\vx)\Vert_2} + 1$, 
\begin{equation}
  \Vert\vdelta^\top\nabla g^y(\vx)\Vert_2 + \Vert R_1(\vx)\Vert_2 = c\Vert\vdelta^\top\nabla g^y(\vx)\Vert_2,
\end{equation}
and it would be the worst-case for the linear assumption when $\vdelta=\varepsilon\bm{v}_{max}$. By taking $\varepsilon\bm{v}_{max}$ as $\vdelta$, $\Vert R_1(\vx)\Vert_2$ can be estimated by
\begin{equation}
  \max\left\{0, \Vert g^y(\vx + \varepsilon\bm{v}_{max}) - g^y(\vx)\Vert_2 - \Vert\varepsilon\bm{v}_{max}^\top \nabla g^y(\vx)\Vert_2\right\}. \label{eqn:est_r1}
\end{equation}
Since $\Vert g^y(\vx + \varepsilon\bm{v}_{max}) - g^y(\vx)\Vert_2 - \Vert\varepsilon\bm{v}_{max}^\top \nabla g^y(\vx)\Vert_2 \leq \Vert R_1(\vx)\Vert_2$. Putting Eq.~(\ref{eqn:est_r1}) into $c$ and using the result in Eq.~(\ref{eqn:ub_l2}), we have
\begin{align}
  c &= \max\left\{0, \frac{\Vert g^y(\vx + \varepsilon\bm{v}_{max}) - g^y(\vx)\Vert_2 - \Vert\varepsilon\bm{v}_{max}^\top \nabla g^y(\vx)\Vert_2}{\xi_{max}\varepsilon}\right\} + 1\\
  &=\max\left\{1, \frac{\Vert g^y(\vx + \varepsilon\bm{v}_{max})-g^y(\vx)\Vert_2 }{\xi_{max}\varepsilon}\right\}.
\end{align}

\section{Analysis of attribution gradients}\label{append:analyse_attr_grad}
\subsection{The gradient of integrated gradients}
We provide the justification showing that the gradient of IG is diagonal-dominated. Consider that
    \begin{equation}
        \text{IG}(\vx)_i = x_i\times\frac{1}{m}\sum_{\alpha=1}^m \frac{\partial f(\frac{\alpha}{m}\vx)}{\partial x_i}
    \end{equation}
    and
    \begin{equation}
        \nabla \text{IG}(\vx)_{ij} = \frac{\partial \text{IG}(\vx)_i}{\partial x_j}
    \end{equation}
    If $i\neq j$, then
    \begin{equation}
        \frac{\partial \text{IG}(\vx)_i}{\partial x_j} = x_i\cdot \frac{1}{m}\sum_{\alpha=1}^m\frac{\partial^2 f(\frac{\alpha}{m}\vx)}{\partial x_i\partial x_j}\times\frac{\alpha}{m}
    \end{equation}
    If $i=j$, then
    \begin{equation}
        \frac{\partial \text{IG}(\vx)_i}{\partial x_j} = \frac{1}{m}\sum_{\alpha=1}^m \frac{\partial f(\frac{\alpha}{m}\vx)}{\partial x_j} + x_i\cdot \frac{1}{m}\sum_{\alpha=1}^m\frac{\partial^2 f(\frac{\alpha}{m}\vx)}{\partial x_i\partial x_j}\times\frac{\alpha}{m}
    \end{equation}
    Denote that $H^{(\alpha)}_{ij} = \frac{\partial^2 f(\frac{\alpha}{m}\vx)}{\partial x_i\partial x_j}$, \emph{i.e.}, $H^{(\alpha)}$ is the Hessian matrix of $f(\frac{\alpha}{m}\vx)$. Thus
    \begin{equation}
        \frac{\partial \text{IG}(\vx)_i}{\partial x_j} = \begin{cases}
            \frac{1}{m}\sum_{\alpha=1}^m \nabla f(\frac{\alpha}{m}\vx) + x_i\cdot \frac{\alpha}{m^2}H^{(\alpha)}_{ij},\quad i=j\\
            x_i\cdot\sum_{\alpha=1}^m\frac{\alpha}{m^2}H^{(\alpha)}_{ij},\quad i\neq j
        \end{cases}
    \end{equation}
    In matrix form, 
    \begin{equation}
        \nabla \text{IG} = \text{diag}\left(\frac{1}{m}\sum_{\alpha=1}^m\nabla f(\frac{\alpha}{m}\vx)\right) + [\vx, \cdots, \vx] \otimes \frac{\alpha}{m^2}\sum_{\alpha=1}^m H^{(\alpha)}\label{apxeqn:grad_IG}
    \end{equation}
    If we use softplus as an activation function, \emph{i.e.}, $g(\vx) = \frac{1}{\beta}\log(1+\exp(\beta\vx))$, then,
    \begin{equation}
        g''(\vx) = \frac{\beta e^{\beta\vx}}{(e^{\beta\vx} + 1)^2}
    \end{equation}
    and 
    \begin{equation}
        \lim_{\beta\rightarrow\infty} g''(\vx) = 0
    \end{equation}
    As $\beta\rightarrow\infty$, $H^{(\alpha)}$ will tend to $0$, and the second term in Eq.~\ref{apxeqn:grad_IG} will tend to 0. At the same time, if we choose the number of steps in IG, $m$ larger, $\frac{\alpha}{m^2}$ will converge to $0$ faster than $\frac{1}{m}$. Therefore, $\nabla\text{IG}$ will be diagonal-dominated.

\subsection{Additional visualization of attribution gradients}\label{append:vis_attgrad}
We provide the first 100-dimensions heatmaps of absolute values of attribution gradients, \emph{i.e.}, gradients of IG, on MNIST and Fashion-MNIST in addition to CIFAR-10 presented in Fig.~\ref{fig:hess_diag}. Moreover, the complete heatmaps for all the three datasets are also presented. As observed in \cref{apxfig:grad_100,apxfig:grad_mnist,apxfig:grad_fmnist,apxfig:grad_cifar}, the matrices of attribution gradients are diagonal-dominant.

\begin{figure}
  \centering
  \begin{subfigure}{.45\textwidth}
    \centering
    \includegraphics[width=\textwidth]{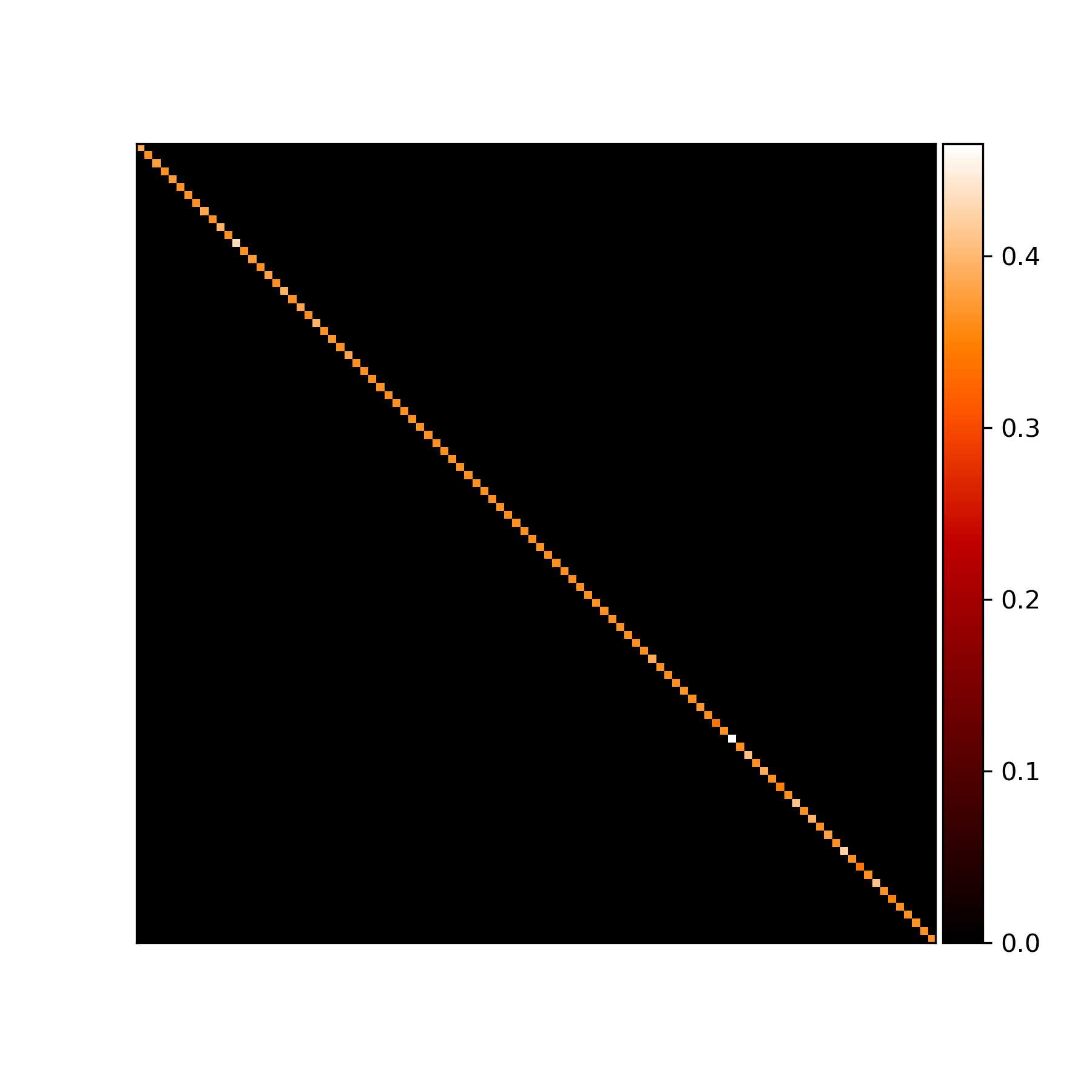}
    \caption{}
  \end{subfigure}
  \begin{subfigure}{.45\textwidth}
    \centering
    \includegraphics[width=\textwidth]{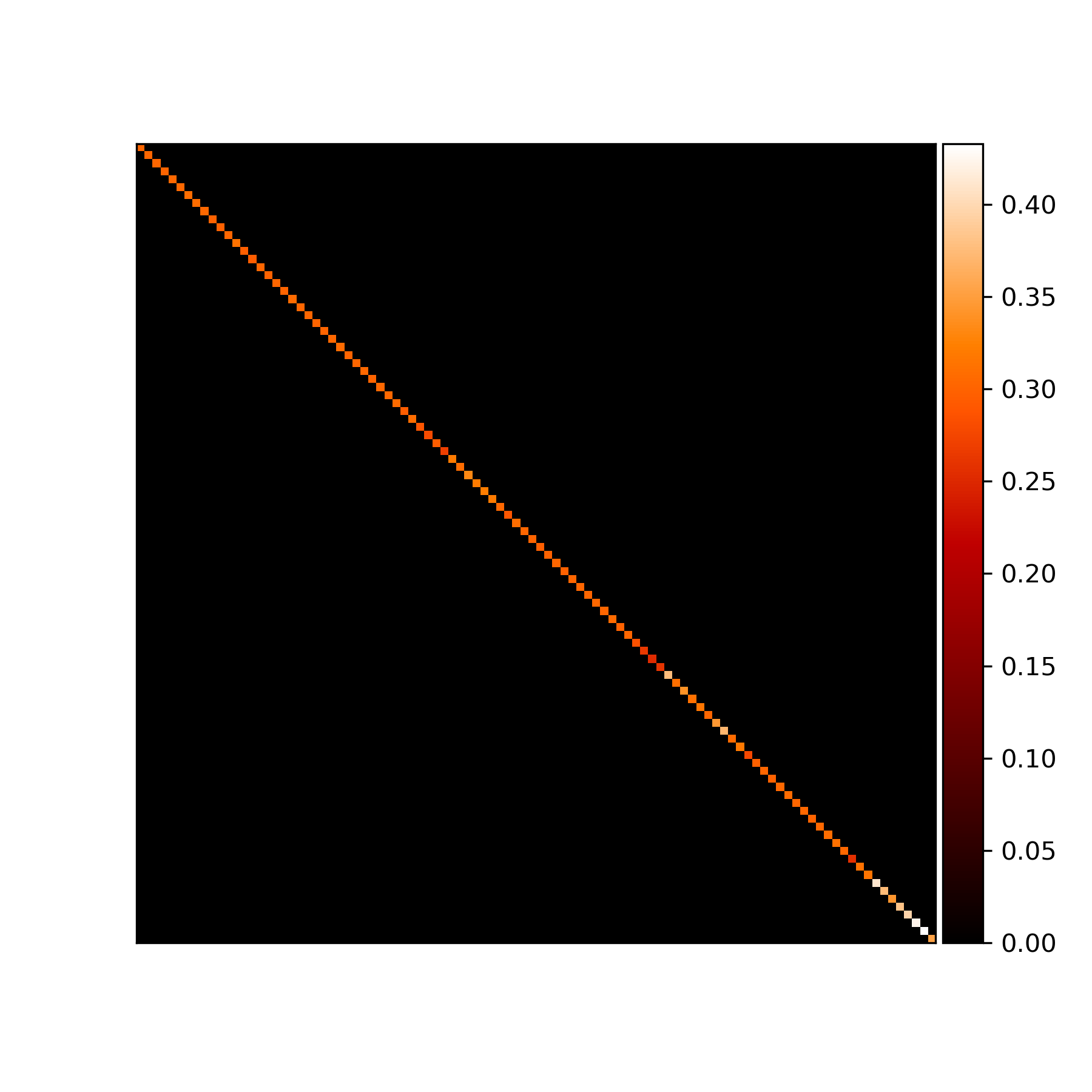}
    \caption{}
  \end{subfigure}
  \caption{The first 100 dimensions of gradient attribution generated from (a) MNIST and (b) Fashion-MNIST.}\label{apxfig:grad_100}
\end{figure}

\begin{figure}  
  \centering
  \includegraphics[width=.8\textwidth]{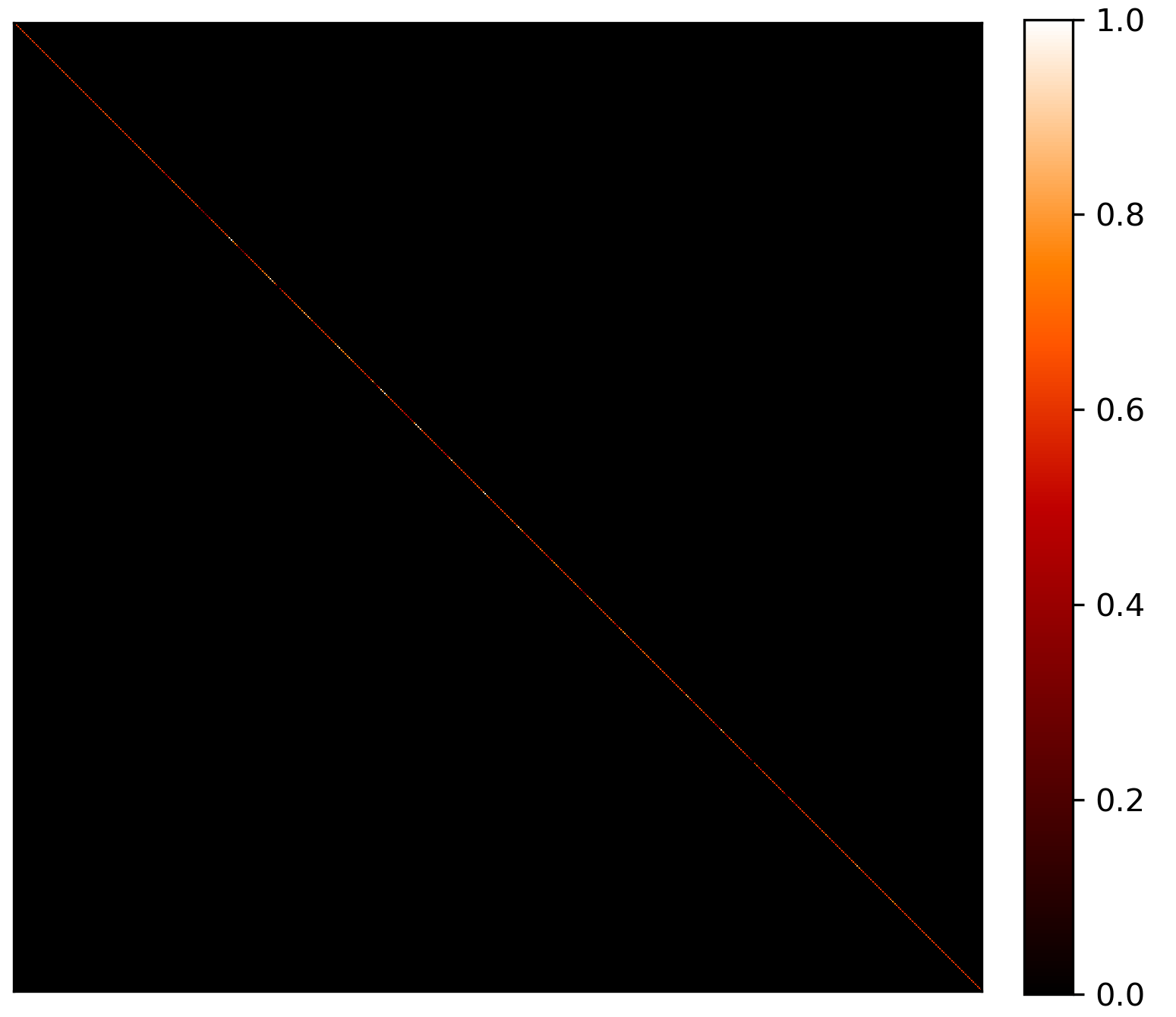}
  \caption{The full heatmap of attribution gradients of MNIST in size $784\times 784$.}\label{apxfig:grad_mnist}
\end{figure}
\begin{figure}
  \centering
  \includegraphics[width=.79\textwidth]{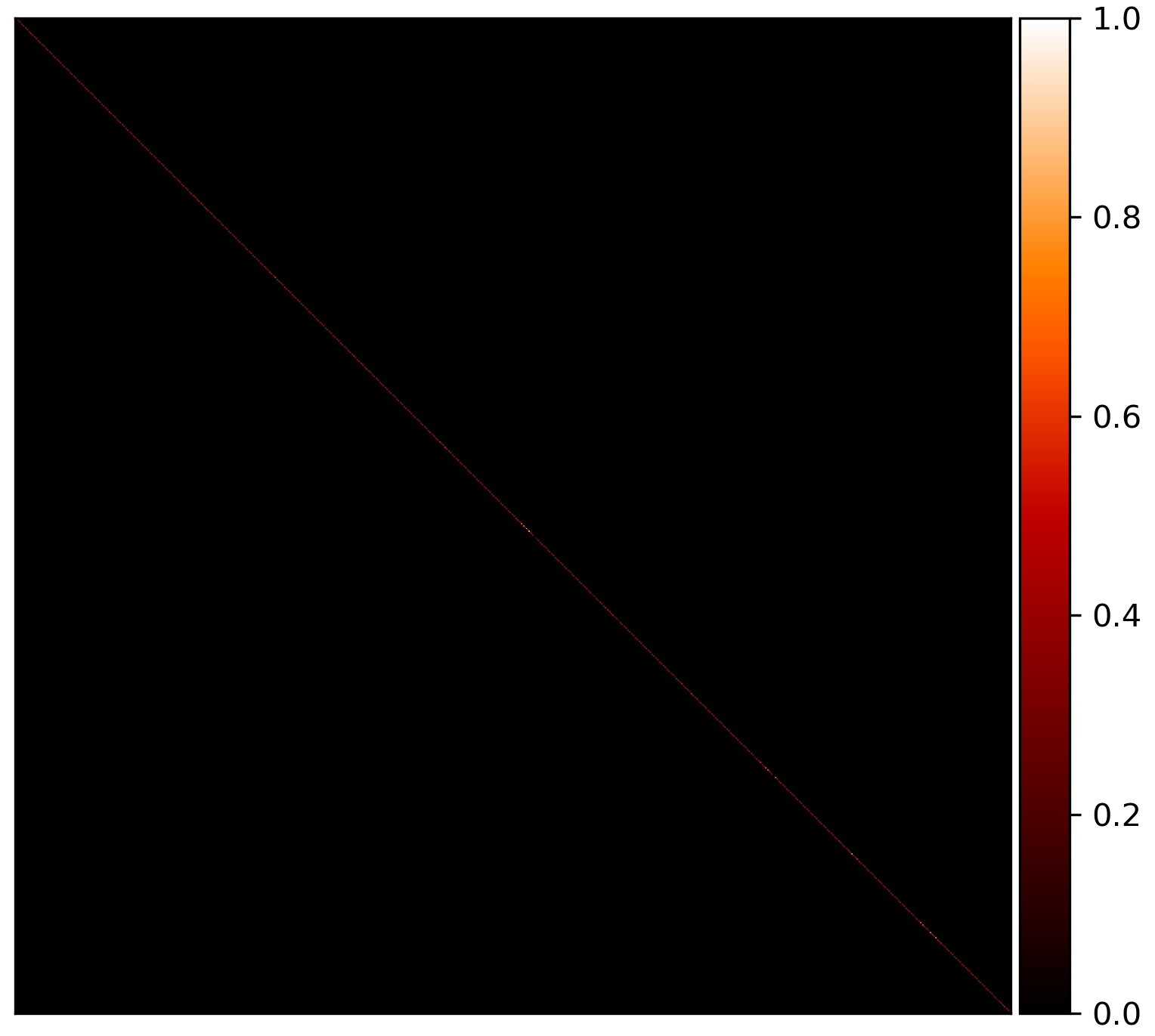}
  \caption{The full heatmap of attribution gradients of Fashion-MNIST in size $784\times 784$.}\label{apxfig:grad_fmnist}
\end{figure}
\begin{figure}  
  \centering
  \includegraphics[width=.8\textwidth]{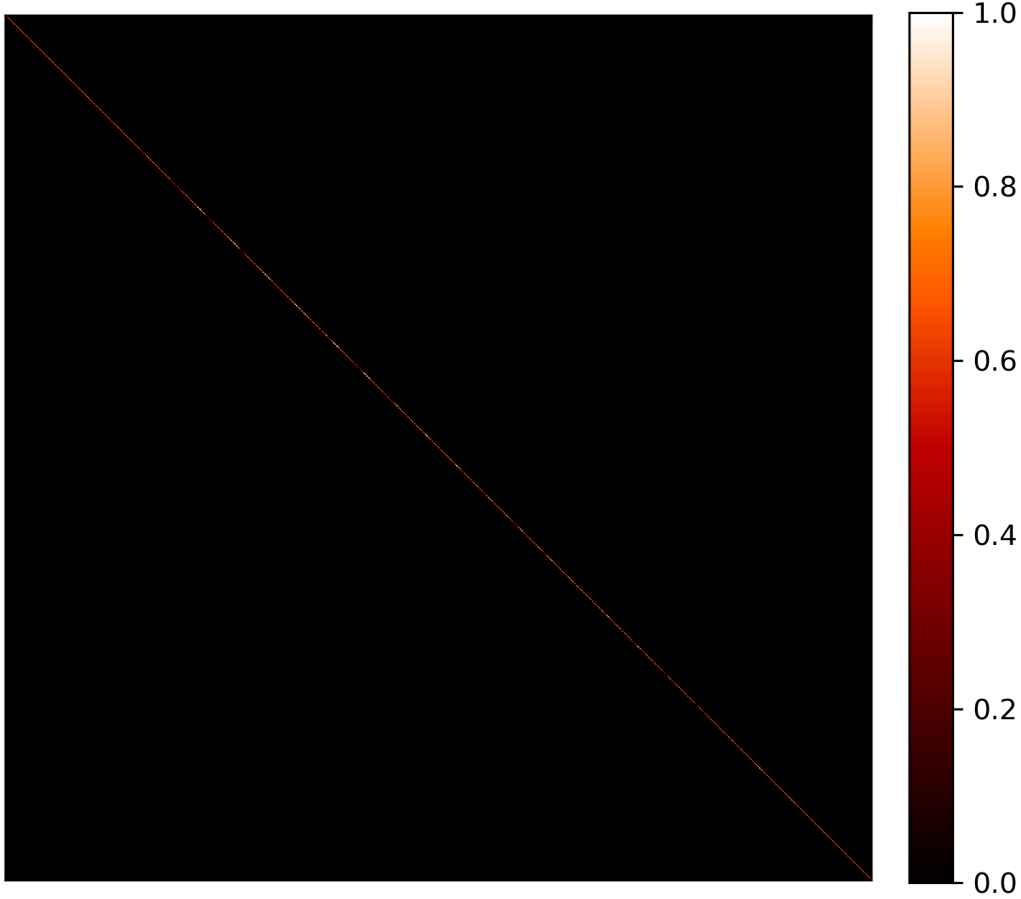}
  \caption{The full heatmap of attribution gradients of CIFAR-100 in size $3072\times 3072$.}\label{apxfig:grad_cifar}
\end{figure}

\section{Additional experimental results}\label{append:additional_exp}
\subsection{Additional results of upper bound on more models without the label constraint}\label{append:exp_unlabel}
In this subsection, we evaluate the proposed upper bound without the label constraint for the other models, apart from TRADES+IGR in the paper. The perturbation size is chosen to be $0.1$ for all evaluations. As in Sec.~\ref{sec:exp}, we use Theorem~\ref{thm:l2_unlabel} and \ref{thm:linf_unlabel} to compute $T_e=\xi_{max}\varepsilon$ and extend it to $T_c$ using Eq.~\ref{eqn:ub_cos}. The modified upper bound $T'_e=c\xi_{max}\varepsilon$ is also provided to address the inaccurate Taylor approximation (less than $1\%$). $\widehat{T}_e$ and $\widehat{T}_c$ are computed from the corresponding average attribution differences. The results are given in Table~\ref{apxtbl:eval_unlabel}. It is shown that the sample distances under both Euclidean and cosine metrics are bounded by $T'_e$ and $T_c$ as expected. All the distortion caused by the attacks \emph{i.e.}, $\widehat{T}_e$ and $\widehat{T}_c$ are smaller than $T'_e$ and $T_c$.
\begin{table}
  \centering
  \caption{Evaluation of upper bounds without the label constraint. The cosine distance values $\widehat{T}_c$ and $\widehat{T}'_c$ are converted to degrees for easier comparison.}\label{apxtbl:eval_unlabel}
  \resizebox{\textwidth}{!}{%
  \begin{tabular}{lccccc|ccccc|ccccc}
    \toprule
     & \multicolumn{5}{c}{SM} & \multicolumn{5}{c}{Input*gradient} & \multicolumn{5}{c}{IG}\\
     \cmidrule(r){2-6}\cmidrule(r){7-11}\cmidrule(r){12-16}
     $\ell_2$ &$\widehat{T}_e$ &$T_e$&$T'_e$&$\widehat{T}_c$ &$T_c$&$\widehat{T}_e$ &$T_e$&$T'_e$&$\widehat{T}_c$ &$T_c$&$\widehat{T}_e$ &$T_e$&$T'_e$&$\widehat{T}_c$ &$T_c$\\
    \midrule
    AT        & 0.44 & 0.94 & 0.98 & 9.19 & 14.87 & 0.07 & 0.63 & 0.63 & 1.17 & 4.34 & 0.04 & 0.25 & 0.25 & 2.73 & 4.77 \\
    IG-NORM   & 0.03 & 0.70 & 0.79 & 4.33 & 9.06 & 0.03 & 0.50 & 0.52 & 1.40 & 4.75 & 0.01 & 0.16 & 0.16 & 1.65 & 4.37 \\
    AdvAAT    & 0.30 & 1.83 & 1.83 & 11.24 & 20.44 & 0.08 & 0.66 & 0.67 & 1.84 & 3.79 & 0.04 & 0.24 & 0.24 & 0.28 & 3.82 \\
    ART       & 0.18 & 0.79 & 0.81 & 10.88 & 14.21 & 0.09 & 0.92 & 0.97 & 0.83 & 6.06 & 0.07 & 0.23 & 0.23 & 0.59 & 4.21 \\
    TRADES    & 0.11 & 0.76 & 0.76 & 10.01 & 18.40 & 0.05 & 0.48 & 0.48 & 1.19 & 3.20 & 0.03 & 0.17 & 0.17 & 1.91 & 3.87 \\
    \midrule
    \midrule
    $\ell_\infty$\\
    \midrule
    AT        & 0.55 & 1.27 & - & 23.47 & 30.18 & 0.63 & 0.73 & - & 9.28 & 61.03 & 0.41 & 0.76 & - & 26.62 & 45.32\\
    IG-NORM   & 0.42 & 0.70 & - & 25.16 & 32.60 & 0.21 & 0.70 & - & 6.88 & 42.94 & 0.20 & 0.48 & - & 21.63 & 35.30\\
    AdvAAT    & 0.64 & 1.83 & - & 25.20 & 31.25 & 0.07 & 0.74 & - & 7.79 & 45.16 & 0.23 & 0.52 & - & 28.73 & 39.40\\
    ART       & 0.49 & 1.01 & - & 23.81 & 35.17 & 0.27 & 0.79 & - & 10.21 & 48.30 & 0.31 & 0.67 & - & 31.01 & 35.64\\
    TRADES    & 0.39 & 0.75 & - & 22.40 & 29.10 & 0.33 & 0.69 & - & 9.17 & 52.63 & 0.23 & 0.50 & - & 22.98 & 36.38\\
    \bottomrule
  \end{tabular}
  }
\end{table}

\subsection{Ablation study of upper bound using different \texorpdfstring{$\varepsilon$}{epsilon}}
In this subsection, we provide more experimental results of the proposed bound on MNIST, Fashion-MNIST and CIFAR-10 in both $\ell_2$ and $\ell_\infty$ cases under label constraint. More specifically, for MNIST and Fashion-MNIST, we additionally provide results of $\varepsilon=0.1$ and $\varepsilon=0.2$ in $\ell_2$ case, and $\varepsilon=0.01$ and $\varepsilon=0.03$ in $\ell_\infty$ case. For CIFAR-10, we provide $\varepsilon=0.2$ and $\varepsilon=0.3$ for $\ell_2$ case, and $\varepsilon=4/255$ and $\varepsilon=8/255$ in $\ell_\infty$ case. The results are presented in \cref{apxtbl:l2_label,apxtbl:linf_label}. For $\ell_2$ constrained case, we also provide the modified upper bound $T'_e$ as in Sec.~\ref*{sec:exp} since the Taylor approximations are inaccurate occasionally ($0\sim 6\%$). For all tested $\varepsilon$, it is noticed that the theoretical bounds bound the sample Euclidean and cosine distance above. In some cases, the means of $T_e$ and $T'_e$ are the same because $T_e$ bound $\widehat{T}_e$ well and the $c$ in Eq.~(\ref{eqn:modified_k}) equals to 1 for $T'_e$. As in Sec.~\ref{sec:exp}, for $\ell_\infty$ case, we do not present the results of $T'_e$, because $T_e$ has bounded all $\widehat{T}_e$ above.

\begin{table}
  \centering
  \caption{Evaluation of $\ell_2$-norm upper bound with the label constraint on MNIST, Fashion-MNIST and CIFAR-10 using different $\varepsilon$.}\label{apxtbl:l2_label}
  \resizebox{.8\textwidth}{!}{%
  \begin{tabular}{lccccc|ccccc}
    \toprule
     &$\widehat{T}_e$ &$T_e$ & $T'_e$ &$\widehat{T}_c$($\deg$) &$T_c$($\deg$)&$\widehat{T}_e$ &$T_e$ & $T'_e$ &$\widehat{T}_c$($\deg$) &$T_c$($\deg$)\\%&$\widehat{T}_e$ &$T_e$ &$\widehat{T}_c$($\deg$) &$T_c$($\deg$)\\
    \cmidrule(r){2-6}\cmidrule(r){7-11}%
    MNIST & \multicolumn{5}{c}{$\varepsilon=0.1$} & \multicolumn{5}{c}{$\varepsilon=0.2$}\\% & \multicolumn{4}{c}{0.3}\\
     \midrule
    AT         & 0.0856 & 0.3074 & 0.3101 & 4.6026 & 14.3020 & 0.1176 & 0.4611 & 0.4617 & 5.9845 & 29.6082  \\
    IG-NORM    & 0.1436 & 0.5776 & 0.5776 & 3.9514 & 14.6430 & 0.2094 & 0.8664 & 0.8679 & 5.4824 & 30.3707  \\
    AdvAAT     & 0.0938 & 0.7182 & 0.7193 & 2.1315 & 13.8325 & 0.1346 & 1.0773 & 1.1013 & 2.8725 & 28.5660  \\
    ART        & 0.2031 & 0.6538 & 0.6542 & 6.4244 & 13.9011 & 0.2302 & 0.9807 & 0.9993 & 8.5982 & 28.7175  \\
    TRADES     & 0.2159 & 1.0120 & 1.0812 & 3.4791 & 14.1049 & 0.3281 & 1.5180 & 1.5211 & 4.9429 & 29.1695  \\
    TRADES+IGR & 0.2171 & 0.9928 & 1.0101 & 3.4171 & 14.0621 & 0.3032 & 1.4892 & 1.4892 & 4.5166 & 29.0745  \\
    \midrule
    \midrule
    Fashion-MNIST& \multicolumn{5}{c}{$\varepsilon=0.1$} & \multicolumn{5}{c}{$\varepsilon=0.2$}\\% & \multicolumn{4}{c}{0.3}\\
    \midrule
   AT         & 0.1080 & 0.1400 & 0.1401 & 16.7770 & 26.6451 & 0.1413 & 0.2100 & 0.2119 & 21.3901 & 63.7570  \\
   IG-NORM    & 0.1232 & 0.3578 & 0.3578 & 8.9312 & 17.8256 & 0.1771 & 0.5367 & 0.5371 & 12.5177 & 37.7516  \\
   AdvAAT     & 0.1500 & 0.3470 & 0.3533 & 7.3499 & 19.0014 & 0.1984 & 0.5205 & 0.5209 & 9.4643 & 40.6308  \\
   ART        & 0.2057 & 0.2774 & 0.2775 & 11.6920 & 19.9515 & 0.2343 & 0.4161 & 0.4161 & 13.4216 & 43.0352  \\
   TRADES     & 0.0797 & 0.1926 & 0.1987 & 10.5544 & 24.7845 & 0.1050 & 0.2889 & 0.2889 & 13.8358 & 56.9729  \\
   TRADES+IGR & 0.0672 & 0.0906 & 0.0906 & 11.3338 & 17.9020 & 0.0879 & 0.1359 & 0.1510 & 14.7998 & 37.9358  \\
   \midrule
   \midrule
   CIFAR-10& \multicolumn{5}{c}{$\varepsilon=0.2$} & \multicolumn{5}{c}{$\varepsilon=0.3$}\\% & \multicolumn{4}{c}{0.3}\\
     \midrule
    AT         & 0.0607 & 0.5064 & 0.5064 & 3.7975 & 9.5783 & 0.0858 & 1.2661 & 1.2661 & 5.2981 & 24.5816  \\
    IG-NORM    & 0.0123 & 0.3164 & 0.3164 & 1.4311 & 8.7679 & 0.0592 & 0.7910 & 0.7910 & 6.9460 & 22.4006  \\
    AdvAAT     & 0.0300 & 0.4772 & 0.4775 & 1.7094 & 7.6575 & 0.0548 & 1.1933 & 1.1933 & 3.0553 & 19.4588  \\
    ART        & 0.0501 & 0.4556 & 0.4699 & 3.1004 & 8.4476 & 0.0718 & 1.1391 & 1.1420 & 6.3493 & 21.5468  \\
    TRADES     & 0.0360 & 0.3468 & 0.3468 & 3.9435 & 7.7550 & 0.0528 & 0.8671 & 0.8780 & 5.7514 & 19.7151  \\
    TRADES+IGR & 0.0395 & 0.3384 & 0.3385 & 4.1222 & 7.6942 & 0.0577 & 0.8460 & 0.8460 & 5.9201 & 19.5551  \\
    \bottomrule
  \end{tabular}
  }
\end{table}

\begin{table}
    \centering
    \caption{Evaluation of upper bounds under $\ell_\infty$-norm constraint and label constraint on MNIST, Fashion-MNIST and CIFAR-10 with different $\varepsilon$.}\label{apxtbl:linf_label}
    \resizebox{.8\textwidth}{!}{%
    \begin{tabular}{lcccc|cccc}
      \toprule
      &$\widehat{T}_e$ &$T_e$ &$\widehat{T}_c$($\deg$) &$T_c$($\deg$)&$\widehat{T}_e$ &$T_e$ &$\widehat{T}_c$($\deg$) &$T_c$($\deg$)\\%&$\widehat{T}_e$ &$T_e$ &$\widehat{T}_c$($\deg$) &$T_c$($\deg$)\\
      \cmidrule(r){2-5}\cmidrule(r){6-9}%
      MNIST & \multicolumn{4}{c}{$\varepsilon=0.01$} & \multicolumn{4}{c}{$\varepsilon=0.03$}\\% & 
      \midrule
      AT          & 0.0556 & 0.1550 & 2.9408 & 7.1839 & 0.0888 & 0.4651 & 4.2516 & 22.0345\\
      IG-NORM     & 0.1005 & 0.2409 & 2.8745 & 6.0632 & 0.1710 & 0.7228 & 4.4179 & 18.4742\\
      AdvAAT      & 0.0608 & 0.4398 & 1.4264 & 5.0839 & 0.1280 & 1.3195 & 2.4883 & 15.4170\\
      ART         & 0.0767 & 0.5644 & 2.8025 & 10.3833 & 0.3617 & 1.6931 & 9.3505 & 32.7312\\
      TRADES      & 0.1634 & 0.4443 & 2.7539 & 6.3323 & 0.3193 & 1.3330 & 4.7523 & 19.3224\\
      TRADES+IGR  & 0.1744 & 0.4077 & 2.7731 & 5.1333 & 0.2932 & 1.2232 & 4.2425 & 15.5702\\
      \midrule
      \midrule
      Fashion-MNIST & \multicolumn{4}{c}{$\varepsilon=0.01$} & \multicolumn{4}{c}{$\varepsilon=0.03$}\\% & \multicolumn{4}{c}{0.3}\\
       \midrule
      AT         & 0.0516 & 0.0560 & 6.5146 & 9.4467 & 0.1043 & 0.1680 & 16.4165 & 29.4979  \\
      IG-NORM    & 0.0611 & 0.1113 & 4.7737 & 8.3315 & 0.1137 & 0.3339 & 8.1315 & 25.7661  \\
      AdvAAT     & 0.0987 & 0.1841 & 5.3706 & 8.1184 & 0.1616 & 0.5523 & 7.9204 & 25.0658  \\
      ART        & 0.0660 & 0.1443 & 6.6582 & 9.2791 & 0.3946 & 0.4329 & 23.0589 & 28.9294  \\
      TRADES     & 0.0509 & 0.0907 & 7.0612 & 9.0233 & 0.0804 & 0.2721 & 10.8579 & 28.0672  \\
      TRADES+IGR & 0.0363 & 0.0505 & 7.1214 & 8.0541 & 0.0716 & 0.1515 & 12.1090 & 24.8550  \\
      \midrule
      \midrule
      CIFAR-10  & \multicolumn{4}{c}{$\varepsilon=4/255$} & \multicolumn{4}{c}{$\varepsilon=8/255$}\\% & \multicolumn{4}{c}{0.3}\\
        \midrule
      AT         & 0.0894 & 0.1200 & 6.0843 & 6.4041 & 0.1549 & 0.2400 & 10.5129 & 12.8901  \\
      IG-NORM    & 0.0388 & 0.0750 & 4.5743 & 5.2004 & 0.0700 & 0.1501 & 8.1882 & 10.4443  \\
      AdvAAT     & 0.0776 & 0.0817 & 2.2657 & 5.7139 & 0.0959 & 0.1635 & 3.8595 & 11.4857  \\
      ART        & 0.0722 & 0.1056 & 4.3010 & 5.2445 & 0.1281 & 0.2113 & 8.4555 & 10.5337  \\
      TRADES     & 0.0539 & 0.0784 & 3.6093 & 5.3381 & 0.0909 & 0.1569 & 9.3571 & 10.7232  \\
      TRADES+IGR & 0.0589 & 0.0821 & 3.8230 & 5.1622 & 0.0978 & 0.1643 & 9.5879 & 10.3668  \\
        \bottomrule
    \end{tabular}
    }
  \end{table}

  \begin{figure}
    \centering
    \begin{subfigure}{.32\textwidth}
      \centering
      \includegraphics[width=\textwidth]{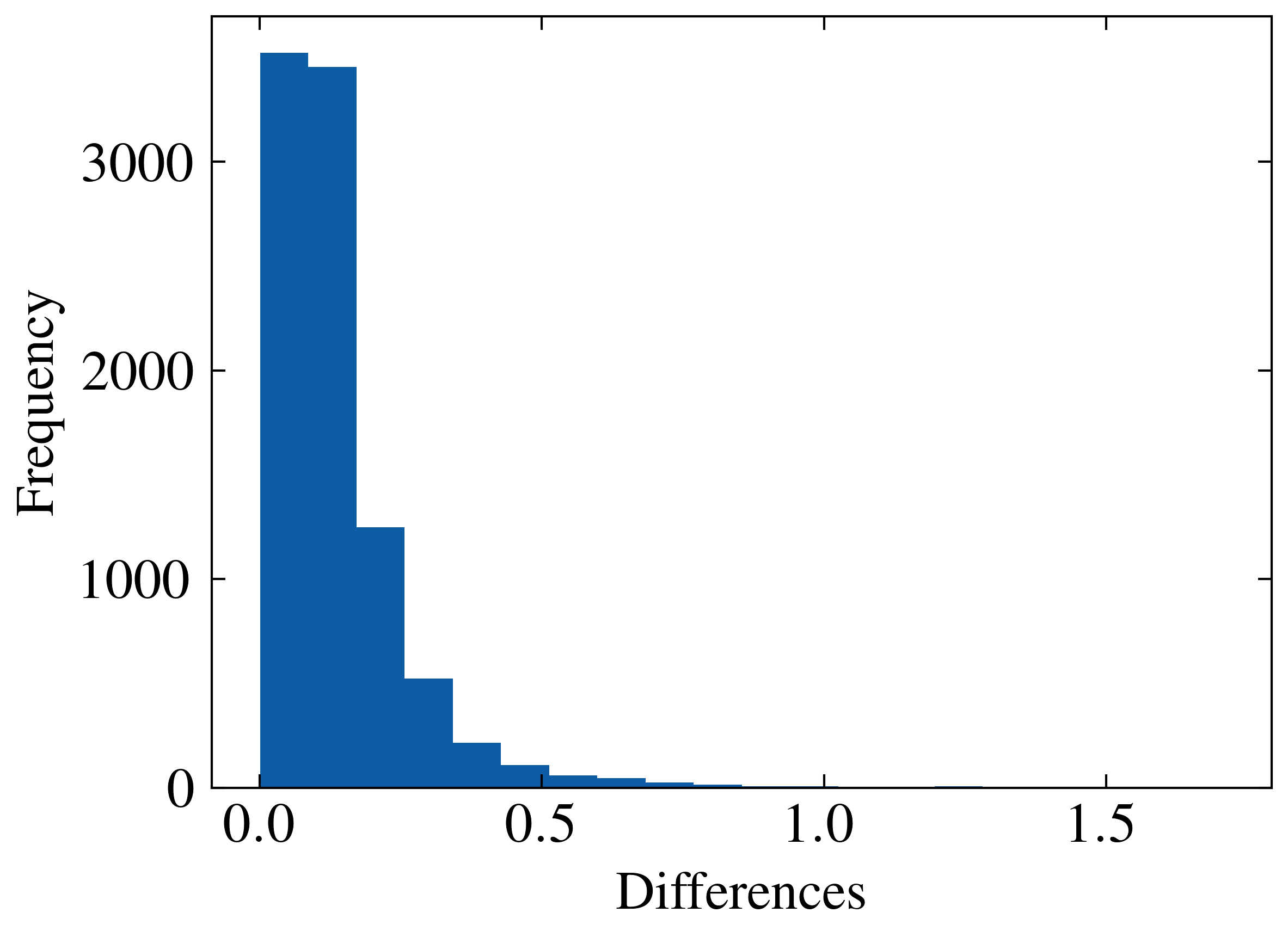}
      \caption{AT}
    \end{subfigure}
    \begin{subfigure}{.32\textwidth}
      \centering
      \includegraphics[width=\textwidth]{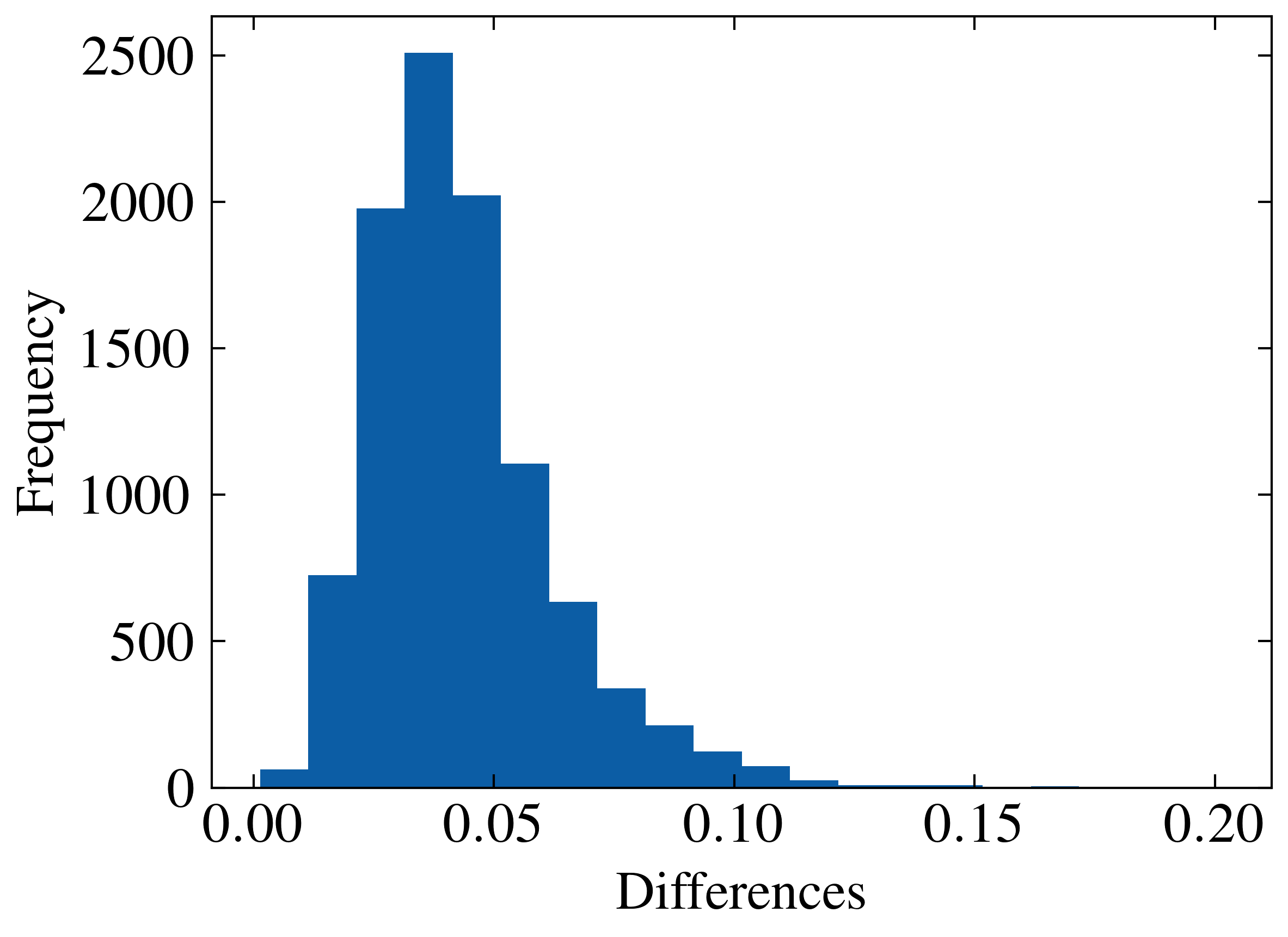}
      \caption{IG-NORM}
    \end{subfigure}
    \begin{subfigure}{.32\textwidth}
      \centering
      \includegraphics[width=\textwidth]{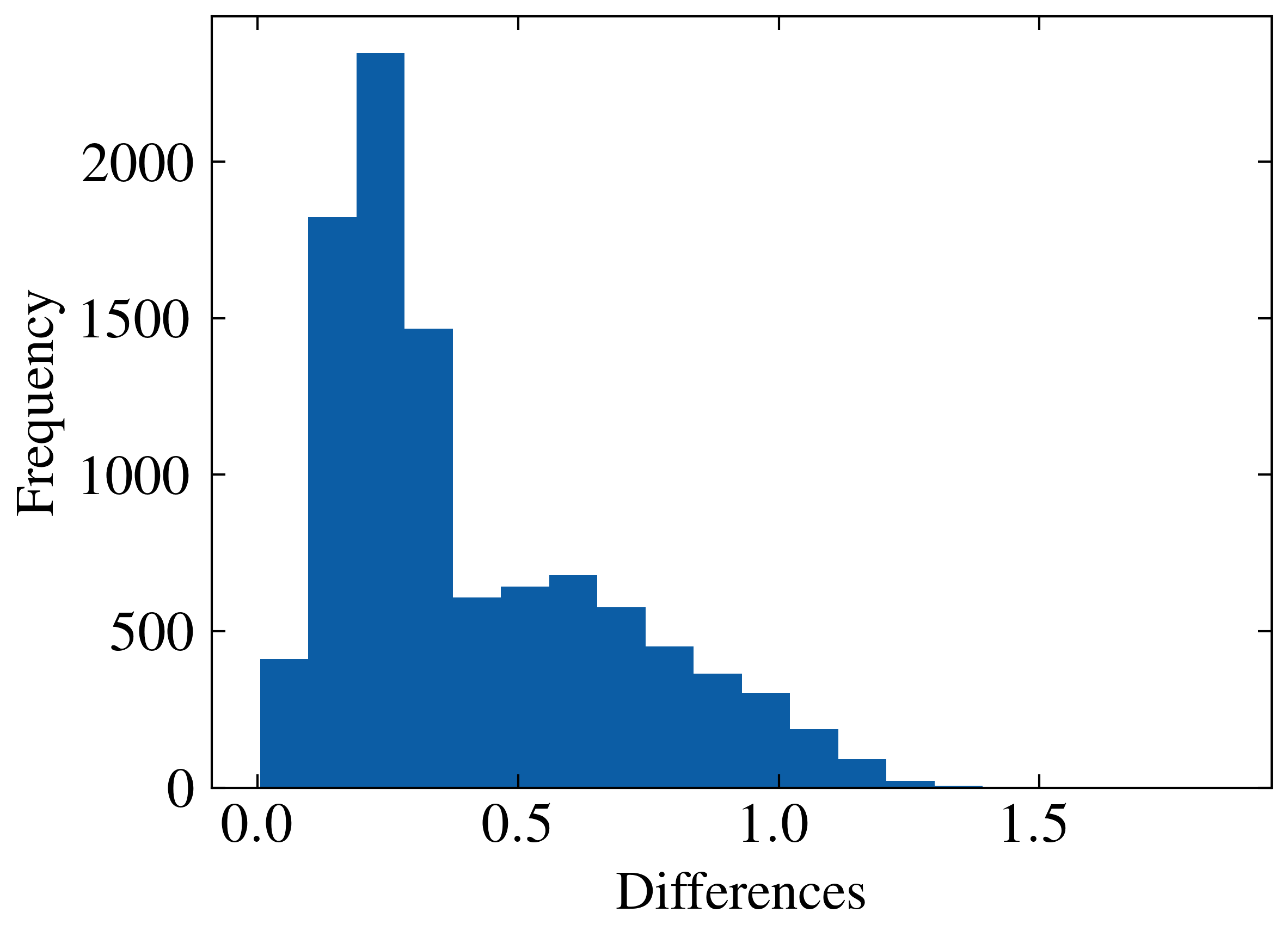}
      \caption{AdvAAT}
    \end{subfigure}
    \begin{subfigure}{.32\textwidth}
      \centering
      \includegraphics[width=\textwidth]{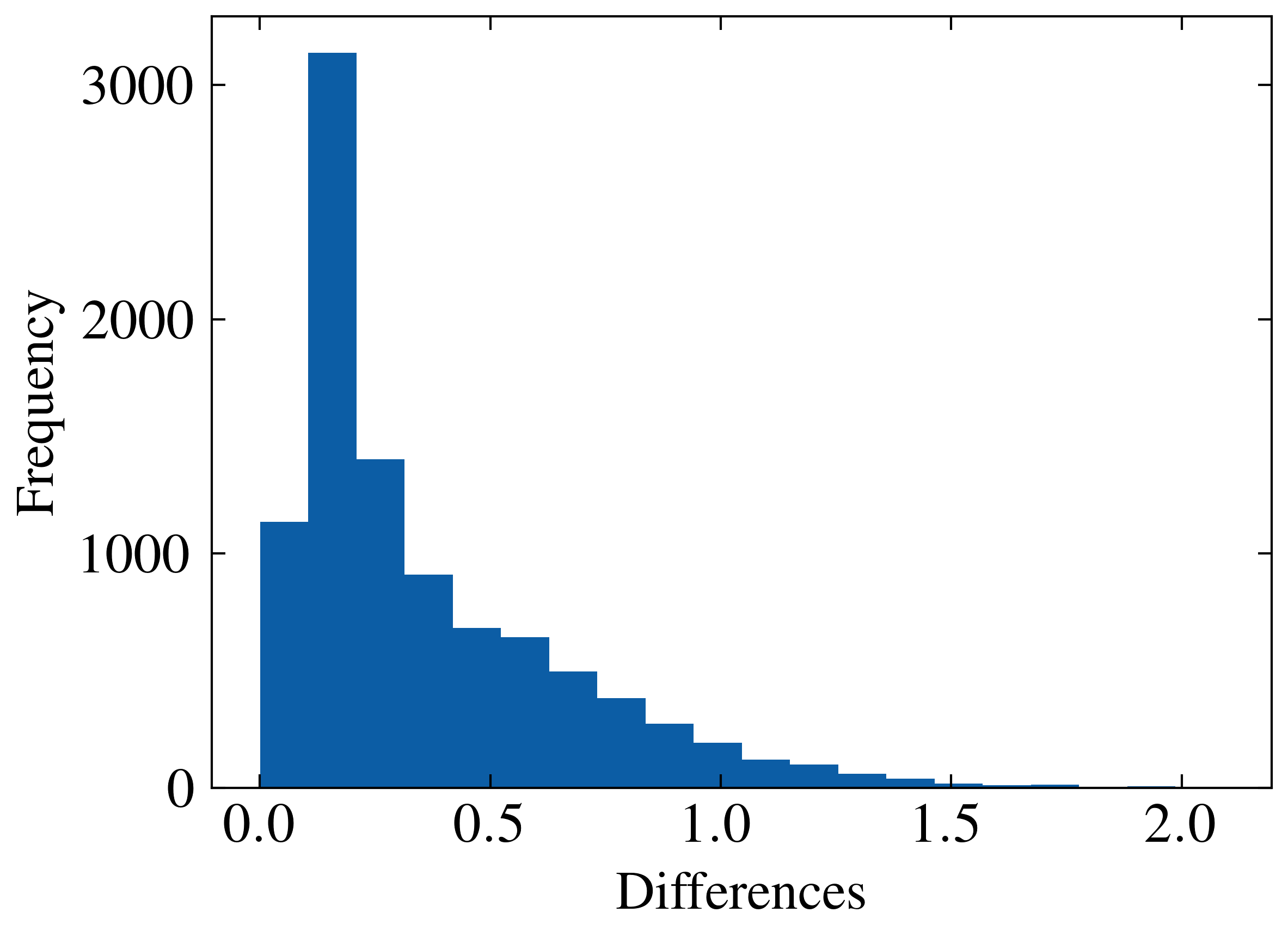}
      \caption{ART}
    \end{subfigure}
    \begin{subfigure}{.32\textwidth}
      \centering
      \includegraphics[width=\textwidth]{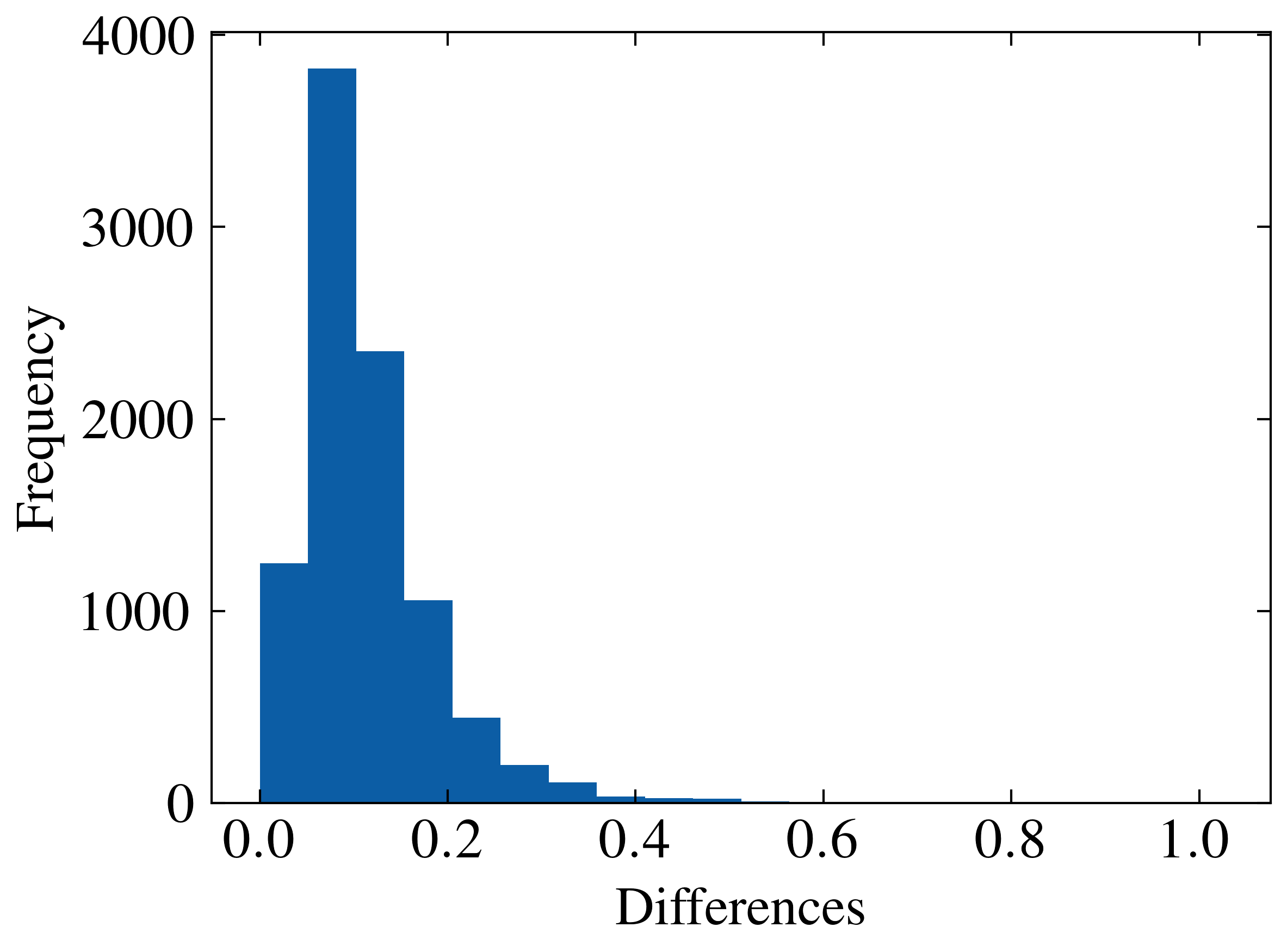}
      \caption{TRADES}
    \end{subfigure}
    \begin{subfigure}{.32\textwidth}
      \centering
      \includegraphics[width=\textwidth]{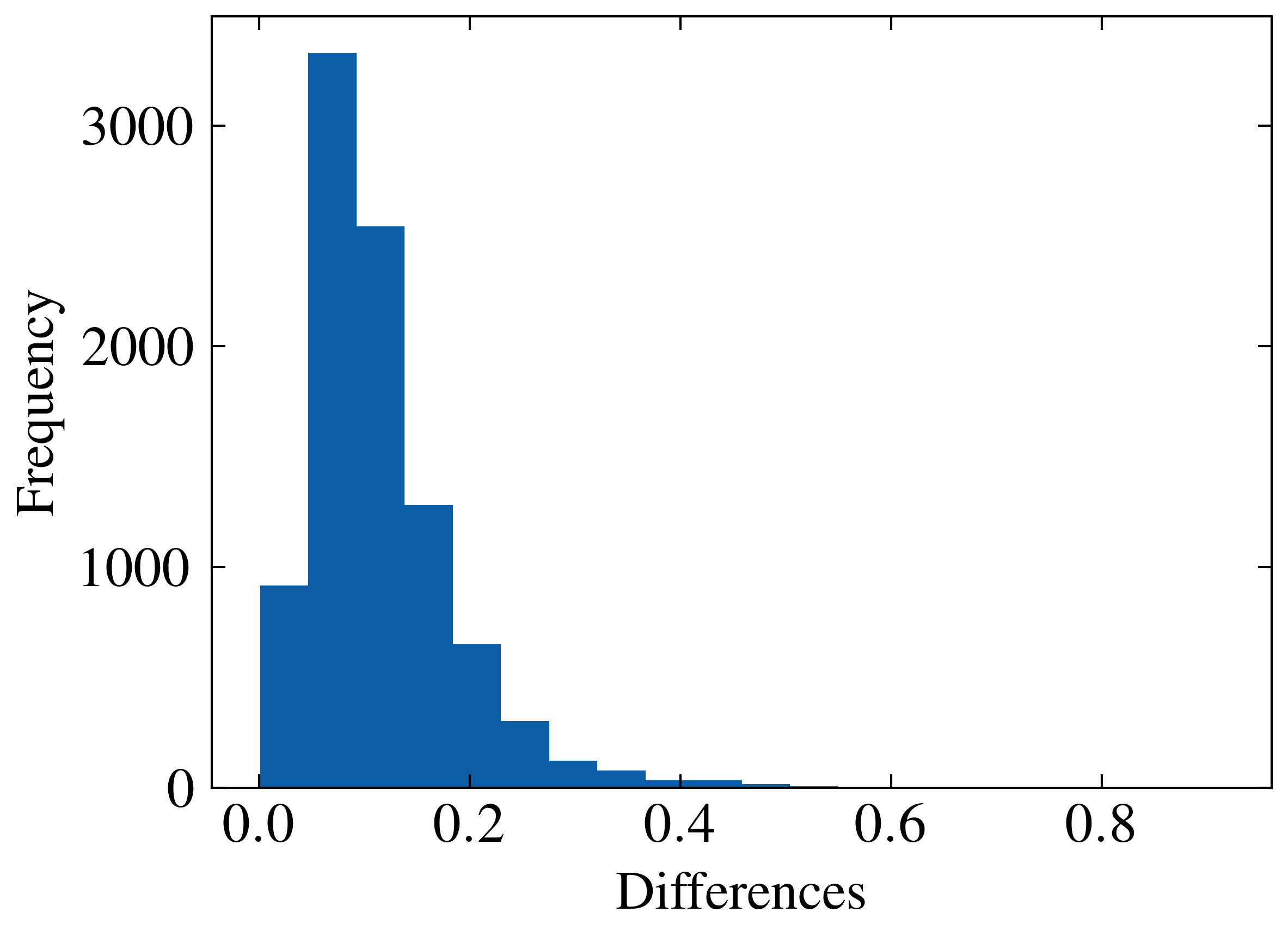}
      \caption{TRADES+IGR}
    \end{subfigure}
    \caption{Distributions of differences between computed bounds and attribution differences from CIFAR-10.}\label{apxfig:dist}
  \end{figure}

  \begin{table}
    \centering
    \caption{Evaluation of upper bounds with the label constraint on Flower dataset. The numbers in the brackets indicate the percentages that attacked attribution is outside the $T_e$.}\label{apxtbl:eval_label_linf_flower}
    \resizebox{.9\textwidth}{!}{%
    \begin{tabular}{lccccc|cccc}
      \toprule
      & \multicolumn{5}{c}{$\ell_2$} & \multicolumn{4}{c}{$\ell_\infty$}\\% & \multicolumn{4}{c}{0.3}\\
       \cmidrule(r){2-6}\cmidrule(r){7-10}%
       &$\widehat{T}_e$ &$T_e$ &$T'_e$ &$\widehat{T}_c$($\deg$) &$T_c$($\deg$)&$\widehat{T}_e$ &$T_e$ &$\widehat{T}_c$($\deg$) &$T_c$($\deg$)\\%&$\widehat{T}_e$ &$T_e$ &$\widehat{T}_c$($\deg$) &$T_c$($\deg$)\\
       \midrule
      AT         & 0.0170 & 0.0341 [2.17\%] & 0.0447 & 1.3165 & 1.9806 & 0.0238 & 0.4100 & 2.1937 & 13.4811 \\
      AdvAAT     & 0.0295 & 0.1424 [0.00\%] & 0.1424 & 1.5568 & 2.2835 & 0.0472 & 0.1025 & 1.4130 & 11.8732 \\
      TRADES     & 0.0220 & 0.0534 [0.72\%] & 0.0592 & 1.3383 & 3.1567 & 0.0182 & 0.1081 & 3.3887 & 11.9829 \\
      TRADES+IGR & 0.0080 & 0.0219 [0.72\%] & 0.0262 & 0.8870 & 2.1255 & 0.0242 & 0.2873 & 1.5930 & 12.5584 \\
      \bottomrule
    \end{tabular}
    }
  \end{table}

  \begin{table}
    \centering
    \caption{Evaluation of upper bounds with the label constraint on ImageNet. The numbers in the brackets indicate the percentages that attacked attribution is outside the $T_e$.}\label{apxtbl:eval_label_linf_imagenet}
    \begin{tabular}{lccccc|cccc}
      \toprule
      & \multicolumn{5}{c}{$\ell_2$}&\multicolumn{4}{c}{$\ell_\infty$}\\% & \multicolumn{4}{c}{0.3}\\
       \cmidrule(r){2-6}\cmidrule(r){7-10}%
       &$\widehat{T}_e$ &$T_e$ &$T'_e$ &$\widehat{T}_c$($\deg$) &$T_c$($\deg$)&$\widehat{T}_e$ &$T_e$ &$\widehat{T}_c$($\deg$) &$T_c$($\deg$)\\%&$\widehat{T}_e$ &$T_e$ &$\widehat{T}_c$($\deg$) &$T_c$($\deg$)\\
       \midrule
       CE &0.1049 &0.1923[3.06\%] &0.2365 &7.7959 &8.0933 &0.3148 &0.7399 &3.8227 &10.9339\\
       AT &0.1077 &0.1588[3.32\%] &0.5221 &3.4773 &5.1797 &0.1974 &0.2226 &0.3455 &8.7333\\
      \bottomrule
    \end{tabular}
  \end{table}

\subsection{Evaluation of upper bounds under \texorpdfstring{$\ell_2$}{L2}-norm and \texorpdfstring{$\ell_\infty$}{Linf}-norm constraints on larger size images.}\label{append:eval_larger}
The proposed method is also scalable to larger size images. In this subsection, we provide experimental results on Flower \cite{nilsback2006visual}, which contains images of size of $128\times 128\times 3$, and a subset of ImageNet \cite{deng2009imagenet} containing 5,000 randomly chosen images with size of $224\times 224\times 3$. We choose $\varepsilon=0.1$ for $\ell_2$ and $\varepsilon=8/255$ for $\ell_\infty$ cases to compute the theoretical upper bounds $T_e$ and $T_c$, as well as the modified bound $T'_e$, as introduced in Sec.~\ref{sec:exp}. The sample distance $\widehat{T}_e$ and $\widehat{T}_c$ are computed from the mean of distances between perturbed and original attributions, where PGD-20 is used as $\ell_2$ attack and 200-step IFIA is used as $\ell_\infty$ attack. In paricular, since the baseline attribution robustness methods do not scale up to ImageNet, we only provide results using standard training and adversarial training to illustrate the scalability of our method. The results are presented in \cref{apxtbl:eval_label_linf_flower,apxtbl:eval_label_linf_imagenet}.

We notice that the theoretical bounds are all valid for larger size images, where all angular and modified Euclidean bound effective bound the maximum discrepancy of perturbed attributions. It worths noting that the computation costs of the values for the upper bound in $\ell_2$-norm constrained case become heavier for high-dimensional images due to the computation of eigenvalues for large matrices. For $\ell_\infty$-norm case, these eigenvalue computations have been avoided. We will study the scalability of our methods under $\ell_2$-norm constraint in future work.

\section{Alternative formulation of upper bound to the worst-case attribution deviations}\label{append:alt_form}
The formulation of Eq.~\ref{eq:opt_cert_ar} can be rewritten in an equivalent form to find the maximum $\varepsilon$ subject to the attribution difference under certain threshold $\omega$. Formally, the formulation can be written as
\begin{equation}
  \begin{aligned}
    \max \quad&\varepsilon\\
    \text{s.t.}\quad &D(g^y(\vx), g^y(\vx+\vdelta))\leq \omega\\
    & \Vert\vdelta\Vert_p\leq\varepsilon\\
    & \argmax_k f_k(\vx) = \argmax_k f_k(\vx+\vdelta)
  \end{aligned}
\end{equation}

Under the above formulation, we can use the theoretical bound derived using Eq.~\ref{eq:opt_cert_ar} to find the corresponding optimal $\varepsilon$. For the $\ell_2$-norm case with or without the label constraint, when $D(\cdot, \cdot)$ is the $\ell_2$ distance, the maximum $\varepsilon$ can be computed using the upper bound $\xi_{max}\varepsilon$ derived in Theorem~\ref{thm:l2_unlabel}, 
\begin{align}
    \max_{\vdelta} \Vert g^y(\vx+\vdelta) - g^y(\vx)\Vert_2 = \xi_{max}\varepsilon\leq \omega\\
    \Rightarrow \varepsilon \leq \frac{\omega}{\xi_{max}}
\end{align}
Similarly, the maximum $\varepsilon$ when $D(\cdot, \cdot)$ is cosine distance can be derived using Corollary~\ref{cor:cert_cos_sim} as
\begin{align}
  \max_{\vdelta} D_c(g^y(\vx+\vdelta), g^y(\vx)) = 1 - \sqrt{1-\frac{\xi_{max}\varepsilon}{\Vert g^y(\vx)\Vert_2^2}}\leq \omega\\
  \Rightarrow\varepsilon \leq \frac{\Vert g(\vx)\Vert_2^2}{\xi_{max}}\left(1-(1-\omega)^2\right)
\end{align}

The maximum $\varepsilon$ for the $\ell_\infty$ constraint case with and without the label constraint can be also derived in the same way using the relaxed upper bound in Theorem~\ref{thm:linf_unlabel}. Since the Kendall's rank correlation is discontinuous, researchers proposed to use cosine similarity and $\ell_p$ distance to measure the similarity/dissimilarity between attributions from attacked samples and original samples \cite{wang2022exploiting,chen2019robust,boopathy2020proper}. Thus, in this work, we derive the bounds for cosine similarity and Euclidean distance.
\end{document}